\documentclass{article} 
\usepackage{iclr2027_conference,times}

\usepackage{amsmath,amsfonts,bm}

\def\eqref#1{equation~\ref{#1}}

\def\1{\bm{1}}

\DeclareMathAlphabet{\mathsfit}{\encodingdefault}{\sfdefault}{m}{sl}
\SetMathAlphabet{\mathsfit}{bold}{\encodingdefault}{\sfdefault}{bx}{n}

\usepackage{float}
\usepackage{hyperref}
\usepackage{booktabs}
\usepackage{url}
\usepackage{wrapfig} 
\usepackage{graphicx}
\usepackage{amssymb} 
\usepackage{enumitem}
\usepackage{amsthm}
\usepackage{comment}

\newtheorem{theorem}{Theorem}

\title{Learning continuous patient trajectories from
electronic health records}
\hypersetup{hidelinks}

\usepackage{titlesec}
\newcommand{\AppendixBanner}{%
  \titleformat{\section}
    {\large\bfseries}
    {\thesection}{0.6em}{}
  \noindent\rule{\textwidth}{2pt}\vspace{2mm}
  \begin{center}
    {\LARGE\bfseries SUPPLEMENTARY INFORMATION}
  \end{center}
  \noindent\rule{\textwidth}{1pt}\vspace{2mm}
}

\fancypagestyle{firstpage}{
  \fancyhead{}
  \fancyhead[L]{Preprint.}
  \renewcommand{\headrulewidth}{1pt}
}

\fancypagestyle{mainpages}{
  \fancyhead{}
  \fancyhead[C]{Learning continuous patient trajectories from electronic health records}
  \renewcommand{\headrulewidth}{1pt}
}

\author{
\normalfont
\begin{tabular}{@{}l@{}}
\textbf{Silas Ruhrberg Estévez}$^{1,2}$\thanks{Corresponding author: \texttt{sr933@cam.ac.uk}}
\quad
\textbf{Kara Liu}$^{2}$
\quad
\textbf{Christopher Chiu}$^{1}$
\quad
\textbf{Benjamin Atta Owusu}$^{3}$
\\[2pt]
\textbf{Umesh Kadam}$^{3}$
\quad
\textbf{Russ B. Altman}$^{2}$\thanks{Equal co-advising}
\quad
\textbf{Mihaela van der Schaar}$^{1}$\footnotemark[2]
\end{tabular}
\\[10pt]
\hspace{1pt}
$^{1}$University of Cambridge
\quad
$^{2}$Stanford University
\quad
$^{3}$University of Exeter
}

\begin{document}

\iclrfinaltrue

\pagestyle{mainpages}
\maketitle
\thispagestyle{firstpage}

\begin{abstract}
Electronic health records provide irregular observations of latent patient states that evolve continuously over time. Recent autoregressive models condition on clinical histories to forecast future events as sequences of discrete observations. Conversely, multi-marginal flow matching provides a continuous-time formulation, but using multiple observations to supervise training paths does not itself give the learned dynamics access to preceding patient history. We introduce \texttt{EHRFlow}, a multi-marginal flow-matching framework that conditions on encoded patient history, thereby allowing future dynamics to depend on the patient's prior clinical trajectory. Our proposed framework accommodates irregular observation times and supports forecasting at arbitrary horizons. Across controlled synthetic benchmarks, \texttt{EHRFlow} improves clinical-code forecasting and latent-state recovery. On real-world clinical datasets comprising more than one million patients, including an independent external validation cohort, \texttt{EHRFlow} improves horizon-averaged top-5 clinical-code accuracy over autoregressive and history-independent flow-matching baselines. Finally, in a controlled counterfactual simulation, conditional guidance approximates the known effect of an antihypertensive intervention without training a task-specific outcome model.
\end{abstract}

\section{Introduction}

Forecasting the future health trajectories of individual patients would enable more personalised care by anticipating future patient states and medically relevant events. Recent advancements in autoregressive and generative machine learning have demonstrated the ability to learn from patient medical history using electronic health records (EHRs) to generate sequences of future diagnoses, treatments, and measurements \citep{Choi2016Doctor,Zhang2026Apollo,Shmatko2025}. 

\begin{wrapfigure}{r}{0.5\textwidth}
\vspace{-15pt}
\centering
\includegraphics[width=\linewidth]{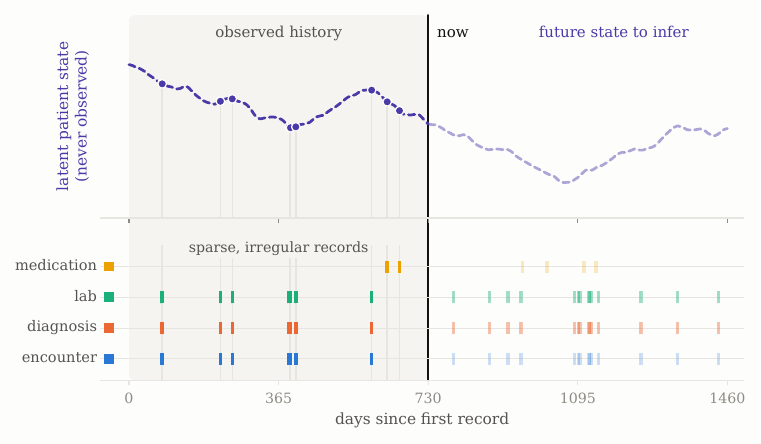}
\vspace{-25pt}
\caption{Schematic of an unobserved, continuously evolving patient state (top) and sparse EHR observations (bottom).}
\vspace{-10pt}
\label{fig:intro}
\end{wrapfigure}

Despite these developments, trajectory forecasting remains an unresolved problem in clinical machine learning \citep{Goldstein2016}. Modeling patient trajectories is highly challenging as the data are longitudinal, high-dimensional, mixed-type, and irregularly sampled. Furthermore, patients differ in observation schedules and rates of clinical progression, without a common temporal reference point. Additionally, forecasting future patient states requires knowledge of their clinical history, as patients in similar current states may have different futures depending on their preceding disease course, treatment exposure, or rate of progression \citep{Swain2003,Coresh2014}. EHRs provide only sparse, partial observations of a patient state that continues to evolve between records. This motivates modeling continuous latent trajectories conditioned on clinical history to infer future states (Figure~\ref{fig:intro}).

Several existing works have tried to address these challenges. Autoregressive models condition on patient history but predominantly represent clinical data as discrete events \citep{Shmatko2025,Waxler2025,Kraljevic2024}. As a result, these approaches do not explicitly represent the continuously evolving patient state between recorded events, and forecasting at a distant horizon typically requires sequential rollout generation, which may accumulate prediction errors. Conversely, continuous-time models capture evolving patient dynamics. Neural ordinary differential equations (ODEs) parameterise dynamics with a neural vector field \citep{Chen2018neural}, and continuous normalizing flows use these dynamics to transport probability distributions over time \citep{Grathwohl2018}.

Flow matching proposes learning the vector field by regressing directly onto target velocities along prescribed probability paths \citep{lipman2023flow}, avoiding numerical integration during training. However, flow matching methods typically rely on paired samples to construct a probability path and require specifying how observations across time are coupled. For instance, in many biological settings, observations are collected as unpaired population snapshots along a shared experimental timeline, thus requiring trajectory reconstruction such as through optimal transport \citep{Schiebinger2019,Tong2020TrajectoryNet}. EHR data, meanwhile, provide the opportunity to leverage multiple observations from the same patient when learning these dynamics. Recent multi-marginal and stream-level flow-matching approaches use multiple observed states to construct training paths \citep{rohbeck2025modeling,wei2025streamlevel}, with applications to longitudinal medical imaging \citep{Islam2025}. However, in these works, learned vector fields are parameterised by only the current patient state and temporal information, without access to the patient's preceding history. Related work by \citet{zhang2024trajectory} demonstrated that conditioning the vector field on the most recent observations can improve performance. However, their work focused on low-dimensional clinical time series data and formulated the target flow using pairs rather than multiple trajectory observations, as in multi-marginal approaches.

We introduce \texttt{EHRFlow}, a history-conditioned multi-marginal flow-matching framework for continuous latent-state forecasting from longitudinal and irregularly sampled EHRs. Our framework uses multiple observations from each patient to supervise velocity during training, and conditions the vector field on encoded patient history and the sampled calendar duration, thus preserving relative observation times. At inference, the field is integrated from the encoded anchor state, conditioned on the preceding history, to forecast an arbitrary horizon.

\textbf{Contributions.} Our core contributions are threefold:
\vspace{-0.5em}
\begin{enumerate}[leftmargin=*,itemsep=0.4ex,parsep=-0.1ex,topsep=0.25ex]
\item \textbf{Conceptually}, we formulate longitudinal EHR forecasting as learning continuous latent patient trajectories and provide a theoretical illustration of why matching population marginals alone does not necessarily recover individual trajectories.
\item \textbf{Methodologically}, we introduce \texttt{EHRFlow}, extending multi-marginal flow matching with explicit history conditioning and a temporal formulation that accommodates irregular observations and arbitrary forecast horizons.
\item \textbf{Empirically}, we demonstrate across both synthetic benchmarks and real-world EHR datasets, comprising more than one million patients and including an independent external validation cohort, that our method improves forecasting performance over autoregressive and continuous-time baselines. We further recover a known treatment effect through intervention-conditioned forecasting in a controlled simulation.
\end{enumerate}

\section{Related Work}

\textbf{Sequence models for clinical data.}
Early RNN-based approaches such as RETAIN \citep{Choi2016Retain} and DoctorAI \citep{Choi2016Doctor} used patient histories to predict clinical outcomes. Transformers enabled richer longitudinal representations, including BEHRT \citep{Li2020}, MOTOR \citep{steinberg2024motor}, and APOLLO \citep{Zhang2026Apollo}. Foresight \citep{Kraljevic2024}, ETHOS \citep{Renc2024}, DELPHI-2M \citep{Shmatko2025}, and COMET \citep{Waxler2025} extend next-token prediction to generate clinical timelines, while SurvivEHR \citep{Gadd2026} uses a competing-risks framework. These models accommodate irregular timing through temporal encodings or time-to-event predictions, but generally represent evolution through discrete events. \texttt{EHRFlow} instead uses longitudinal context to condition a continuously evolving latent patient state.

\textbf{Continuous-time and state-space models.}
Neural ODEs \citep{Chen2018neural} and Latent ODEs \citep{Rubanova2019} learn continuous dynamics, while Neural CDEs \citep{Kidger2020}, GRU-ODE-Bayes \citep{DeBrouwer2019} and ContiFormer \citep{chen2023contiformer} incorporate irregular observations into evolving hidden states. For models trained through numerical integration, learning from long records across large patient populations can be computationally demanding. Attentive state-space models capture dependence on preceding latent states \citep{Alaa2019attentive}, while TrajSurv learns continuous latent trajectories for survival prediction \citep{Agrawal2025}. \texttt{EHRFlow} instead learns from directly supervised velocities along paths constructed from multiple observations, avoiding trajectory integration during flow training.

\begin{figure}[t]
    \centering
    \vspace{-5pt}\includegraphics[width=\linewidth]{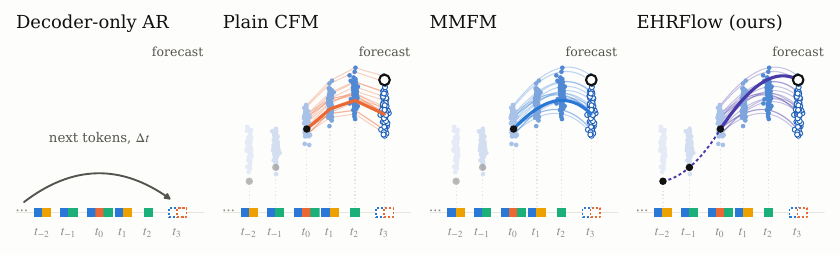}
    \vspace{-25pt}
    \caption{
    Autoregressive models condition on the observed sequence but predict future
    events discretely. Conditional flow matching (CFM) models have pairwise latent
    transitions. Multi-marginal flow matching (MMFM) uses multiple observations
    to define the training trajectory, and \texttt{EHRFlow} additionally conditions the
    learned vector field on the preceding patient history.
    }
    \label{fig:ehrflow_overview}
    \vspace{-12pt}
\end{figure}

\textbf{Flow matching for continuous dynamics.}
Flow matching learns continuous vector fields through simulation-free regression onto conditional velocity targets \citep{lipman2023flow,albergo2023building}, with extensions incorporating data geometry \citep{kapusniak2024metric} and applications to probabilistic forecasting \citep{kollovieh2025flow}. Trajectory Flow Matching \citep{zhang2024trajectory} conditions on only the most recent observations for clinical time-series modelling, using pairwise flow supervision. Multi-marginal \citep{rohbeck2025modeling} and Gaussian-process \citep{wei2025streamlevel} formulations instead use multiple observations to supervise training paths, but do not condition the learned field on the preceding patient history. \texttt{EHRFlow} builds on these multi-observation path constructions by explicitly conditioning the vector field on the patient history available at the forecast anchor. Its latent formulation supports forecasts of both continuous and discrete clinical tokens at arbitrary horizons. Future observations supervise the training path, while preceding observations provide inference-time context (see Figure~\ref{fig:ehrflow_overview}).
\section{Problem formulation and preliminaries}

\subsection{Observed and latent patient states}

We assume that the clinical state of a patient evolves continuously over time and that EHRs provide sparse, irregular, and incomplete observations of this underlying process. For patient $i$, we observe a time-stamped sequence
\begin{equation}
    \mathcal{X}_i
    =
    \{(\tau_{i,j},x_{i,j})\}_{j=1}^{m_i},
    \qquad
    \tau_{i,1}<\cdots<\tau_{i,m_i},
\end{equation}
where $x_{i,j}\in\mathbb{R}^{d_x}$ denotes the feature vector constructed from patient $i$'s recorded clinical information at time $\tau_{i,j}$. Each feature vector aggregates all categorical event indicators, continuous-valued measurements, and other observation indicators possible to observe at a single time point. We encode each observation in the sequence into a lower-dimensional continuous-valued latent representation,
\begin{equation}
    z_{i,j}=E_\psi(x_{i,j})\in\mathbb{R}^{d_z},
    \qquad
    \mathcal{Z}_i
    =
    \{(\tau_{i,j},z_{i,j})\}_{j=1}^{m_i},
\end{equation}
where $\mathcal{Z}_i$ is the corresponding time-stamped latent sequence. A decoder $D_\phi$ maps latent states to predictions in the observation space.

For patient $i$, let $a\in\{1,\ldots,m_i\}$ index an anchor observation $x_{i,a}$ at time $\tau_{i,a}$. We denote the patient history available at and before the anchor as
\(
    \mathcal{H}_{i,a}
    =
    \{(\tau_{i,j},x_{i,j})\}_{j=1}^{a}.
\)
Given a forecast horizon $\Delta>0$, our objective is to predict the future observed patient state at time $\tau_{i,a}+\Delta$, using the encoded anchor state $z_{i,a}$ and the observed history $\mathcal{H}_{i,a}$. We model this trajectory evolution in latent space using a neural vector field, which specifies the instantaneous rate of change of the latent state. Integrating the field produces a latent forecast, which is then decoded into predictions of EHR observations. For notational simplicity, we now focus on a generic patient trajectory and omit the patient index $i$.

\subsection{Trajectory modeling with conditional flow matching}
We first consider two encoded observations $z_{s_0}$ and $z_{s_1}$, which are sampled from the same patient at calendar time $\tau_{s_0},\tau_{s_1}$ and separated by $\Delta = \tau_{s_1}-\tau_{s_0}$ calendar days. Conditional flow matching \citep{lipman2023flow,albergo2023building} constructs a conditional probability path between them over normalised flow time $s\in[0,1]$. A common construction is
\begin{equation}
    z_s
    =\mu_s(z_{s_0},z_{s_1})+\sigma_s\epsilon,
    \qquad
    \epsilon\sim\mathcal{N}(0,I),
    \qquad
    s\sim\mathcal{U}(0,1),
\end{equation}
where $\mu_s$ and $\sigma_s$ specify the mean interpolation and noise scale. Holding the sampled $\epsilon$ fixed, differentiation with respect to $s$ gives the conditional velocity
\begin{equation}
    u_s
    =
    \frac{\partial\mu_s(z_{s_0},z_{s_1})}{\partial s}
    +
    \frac{\partial\sigma_s}{\partial s}\epsilon.
\end{equation}
For example, the linear interpolant
$\mu_s(z_{s_0},z_{s_1})=(1-s)\cdot z_{s_0}+s\cdot z_{s_1}$ with $\sigma_s=0$ yields
$u_s=z_{s_1}-z_{s_0}$. To accommodate observation pairs separated by different calendar durations, we condition the vector field on $\lambda=\log(1+\Delta)$ and train it to predict the conditional velocity:
\begin{equation}
    \mathcal{L}_{\mathrm{CFM}}
    =
    \mathbb{E}_{(z_{s_0},z_{s_1},\Delta)\sim\pi,s,\epsilon}
    \left[
        \left\|
            v_\theta(z_s,s,\lambda)-u_s
        \right\|_2^2
    \right],
\end{equation}
where $\pi$ is the empirical distribution over all pairs of consecutive observations belonging to the same patient, across all patients. Because longitudinal records provide a natural correspondence between observations, no optimal-transport coupling between patients is required.

Note that normalised flow time $s$ denotes relative position within an observation window, which maps to calendar time $\tau : = \tau_{s_0} + s\cdot\Delta$. The target $u_s$ is a velocity with respect to normalised flow time. For an interpolated $z$, we have
$\mathrm{d}z/\mathrm{d}s =
    \Delta \mathrm{d}z/\mathrm{d}\tau$,
thereby motivating the need to condition the vector field on calendar distance $\Delta$ to distinguish different window durations with respect to calendar days. 

\subsection{Trajectory modeling with multi-marginal flow matching}

Pairwise conditional flow matching constructs a training path using only two observed patient states. However, EHRs contain many intermediate observations that provide necessary information about the underlying clinical dynamics. We therefore adapt multi-marginal flow matching \citep{rohbeck2025modeling} to construct the training path jointly from multiple, irregularly sampled observations. Let us consider $K\geq2$ encoded observations and calendar times from a single patient's trajectory starting at $\tau_j$ as $\{(\tau_k,z_k)\}_{k=j}^{K+j-1}$, with $\tau_j<\cdots<\tau_{K+j-1}$.
Let $\Delta =\tau_{K+j-1}-\tau_{j}$ be the window's span of total calendar days and $\lambda=\log(1+\Delta)$.
We first normalise their time-stamps into normalised flow time:
\begin{equation}
    t_k
    =
    \frac{\tau_k-\tau_j}{\Delta},
    \qquad
    0=t_j<\cdots<t_{K+j-1}=1. 
\end{equation}

Let $\mathcal{W}=\{(t_k,z_k)\}_{k=j}^{K+j-1}$ denote the sampled window. 
To accommodate both deterministic and stochastic constructions of interpolants connecting the observed states, we write
\begin{equation}
    z_s=I_s(\mathcal{W},\xi),
    \qquad
    u_s=\frac{\partial I_s(\mathcal{W},\xi)}{\partial s},
\end{equation}
where $\xi$ denotes the sampled path randomness and is held fixed when differentiating, and the interpolant must pass through the observed states, i.e.,
$I_{t_k}(\mathcal{W},\xi)=z_k$, for $k=j,\ldots,K+j-1$. 
This yields the multi-marginal flow-matching objective
\begin{equation}
    \mathcal{L}_{\mathrm{MMFM}}
    =
    \mathbb{E}_{(\mathcal{W},\Delta) \sim \pi,s,\xi}
    \left[
        \left\|
            v_\theta(z_s,s,\lambda)-u_s
        \right\|_2^2
    \right],
\end{equation}
where $\pi$ denotes the empirical distribution over windows pooled across patients, each patient contributing up to a fixed number of windows drawn uniformly from its time-window candidates.  Note that this formulation is agnostic to the interpolant family.

\section{\texttt{EHRFlow}}

Although multi-marginal formulations incorporate multiple trajectory observations into training paths, this supervision does not itself provide the learned vector field with access to patient history \citep{rohbeck2025modeling}. Related work conditions flow dynamics on recent measurements \citep{zhang2024trajectory}, but models only pairwise transitions in observation space and conditions on the most recent observations rather than the broader longitudinal history.

\paragraph{History conditioning.} To address this gap, we consider the patient's observed history $\mathcal{H}_{a} =\{(\tau_{l},x_{l})\}_{l=1}^{a}$, where $a$ indexes the first observation of the sampled window $\mathcal{W}$, so that $\tau_a=\tau_j$. Thus, the history $\mathcal{H}_a$ includes only observations up to and including the anchored observation. We propose to condition the vector field on a fixed-dimensional representation of the patient's anchored history:
\begin{equation}
    h_{a}
    =
    g_\omega(\mathcal{H}_{a})
    \in\mathbb{R}^{d_h},
\end{equation}
where $g_\omega$ is a learned history encoder. During flow training, $g_\omega$ is frozen and $h_a$ is held fixed along the entire path, independently of $s$.

 \begin{wrapfigure}{r}{0.45\textwidth}
     \centering
     \vspace{-12pt}
     \includegraphics[width=\linewidth]{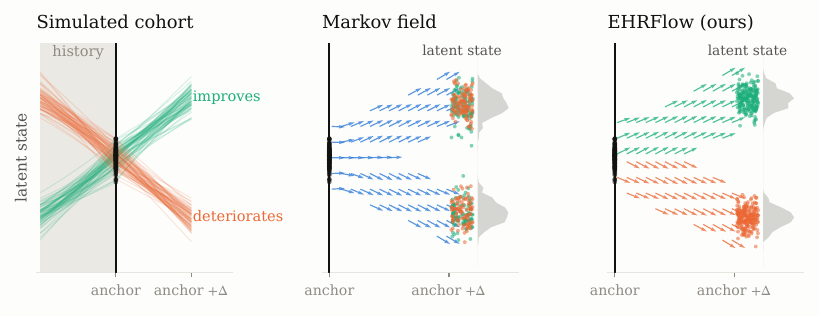}
     \vspace{-20pt}
     \caption{
     Patients with similar current states may have
     different future trajectories depending on how they reached those states.
     }
     \label{fig:history_matters}
     \vspace{-2pt}
 \end{wrapfigure}

We provide theoretical motivation for history-conditioning in Appendix~\ref{app:theory}. First, we show that conditioning on history cannot increase the optimal flow-matching loss, and that the improvement is strict whenever history changes the expected velocity given the current state (Theorem~1).
Second, in a linear-Gaussian example with noisy history, both history-independent and history-conditioned fields recover the correct population marginals, but history-conditioning yields strictly lower individual forecast error. This error decreases as the history becomes more informative and vanishes in the noiseless limit (Theorem~2). 
Figure~\ref{fig:history_matters} illustrates Theorem~2, showing how history conditioning distinguishes trajectories that share the same current state.

\paragraph{Optimisation.}
We pretrain the encoder and decoder using a modified variational autoencoder
objective \citep{Kingma2013}, where the reconstruction loss is adapted for both
categorical and continuous variables. We separately pretrain the history
encoder $g_\omega$ using a GRU \citep{chung2014empirical} to predict the clinical
event recorded at the patient's next observation from the preceding $K$ observations. During flow training, the observation encoder $E_\psi$ and history encoder $g_\omega$ are frozen, while the vector field $v_\theta$ and decoder $D_\phi$ are trained jointly. Training windows begin at a randomly sampled anchor observation and end at the first observation at least $M$ days later, where $M$ is drawn from a fixed set of forecast horizons. We minimise
\begin{equation}
\begin{aligned}
    \mathcal{L}_{\texttt{EHRFlow}}
    =
    \mathbb{E}_{(\mathcal{W},\Delta,\mathcal{H}_a)\sim\pi,\,s,\,\xi}
    \Big[
        &\left\|v_\theta(z_s,s,\lambda,h_a)-u_s\right\|_2^2 \\
        &+\ell_{\mathrm{rec}}\!\left(
            D_\phi\!\left(z_s+(1-s)\cdot v_\theta(z_s,s,\lambda,h_a)\right), \,
            x_{K+j-1}
        \right)
    \Big],
\end{aligned}
\end{equation}
where $x_{K+j-1}$ is the final observation of the window $\mathcal{W}$ and $\hat{z}_1 = z_s + (1-s)\cdot v_\theta(z_s,s,\lambda,h_a)$ is a single-step approximation to the latent endpoint. The reconstruction loss $\ell_{\mathrm{rec}}$ combines cross-entropy for categorical variables and squared error for continuous variables, and helps supervise the vector field and decoder through this endpoint surrogate without numerical integration.

\paragraph{Inference.} At inference, no future observations are required. Given a patient's observable history $\mathcal{H}_{a}$ up until the most recent observation $x_a$ and the requested forecast horizon $\Delta$ in calendar days, we compute the encoded history representation $h_a$ and then integrate the learned history-conditioned vector field over the normalised flow time $s\in[0,1]$:
\begin{equation}
    \frac{\mathrm{d}\hat{z}_s}{\mathrm{d}s}
    =
    v_\theta\!\left(
        \hat{z}_s,
        s,
        \lambda,
        h_{a}
    \right),
    \qquad
    \hat{z}_0=z_a=E_\psi(x_{a}),
\end{equation}
where $\lambda=\log(1+\Delta)$. The ODE solution map $\Phi_\theta$ defines the predicted latent state at calendar time $\tau_a+\Delta$ and corresponding predicted observation state:
\begin{equation}
    \hat{z}_{1} = \Phi_\theta(z_{a},h_{a},\lambda), \qquad \hat{x}(\tau_{a}+\Delta)=D_\phi(\hat{z}_1).
\end{equation}

We report \texttt{EHRFlow} (GP-CFM) for Gaussian-process interpolants \citep{wei2025streamlevel} and \texttt{EHRFlow} (MMFM) for natural cubic-spline interpolants with bridge noise \citep{rohbeck2025modeling}. Full objectives, architectures, and optimisation details are provided in Appendix~\ref{app:ehrflow-impl}.  The training and inference procedure is summarized in Figure~\ref{fig:ehrflow-training} .

\begin{figure}[t]
    \centering
    \vspace{-10pt}\includegraphics[width=\linewidth]{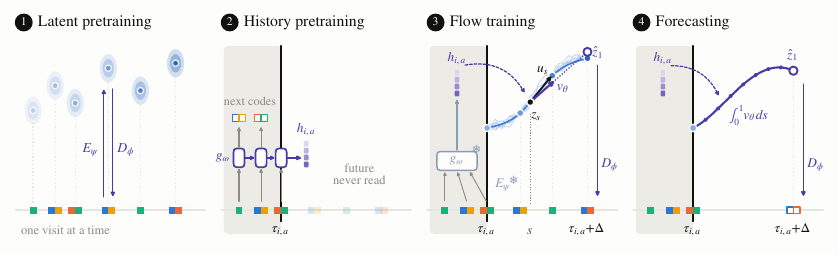}
    \vspace{-20pt}
    \caption{
\textbf{Training and forecasting with \texttt{EHRFlow}.}
(1) An encoder and decoder learn visit-level latent representations.
(2) A history encoder is pretrained on observations up to the forecast anchor.
(3) Both encoders are frozen. Future observations define paths and target velocities for training the history-conditioned field.
(4) At inference, the field is integrated from the encoded anchor and decoded into clinical predictions.
}
    \label{fig:ehrflow-training}
\vspace{-15pt}
\end{figure}
\section{Experiments}

We evaluate \texttt{EHRFlow} across controlled synthetic benchmarks and large-scale longitudinal EHR datasets. Our evaluation is structured around four questions:
\begin{enumerate}[leftmargin=4ex,itemsep=0.4ex,parsep=0ex,topsep=0.25ex]
    \item[(i)] \textbf{Trajectory forecasting:} How accurately can flow-based models forecast future clinical observations compared with autoregressive and continuous-time sequence models?
    \item[(ii)] \textbf{Real-world generalisation:} Do the benefits extend to primary- and secondary-care EHRs, including an independent external cohort?
    \item[(iii)] \textbf{History conditioning and optimisation}: Which aspects of historical conditioning and training contribute to forecasting performance?
    \item[(iv)] \textbf{Intervention modeling:} Can the learned dynamics be conditionally guided to recover a known intervention effect in a controlled simulation?
\end{enumerate}

\subsection{Experimental setup}

\paragraph{Datasets and tasks.}
We evaluate on eight datasets total, spanning from fully synthetic to real-world clinical observations. First, we validate on five controlled synthetic datasets, each with 20,000 simulated patients, where we vary the data-generating mechanism to model five distinct disease trajectories: \textit{Anthracycline}, which models cumulative treatment exposure; \textit{RA on DMARD}, which models delayed treatment response; \textit{CKD}, which models patient-specific progression rates; \textit{RRMS}, which models relapse history; and \textit{Heart failure}, which models a Markovian worsening trajectory (see Appendix \ref{app:synthetic}). Second, we validate on 23,072 patients generated using the clinical simulator Synthea \citep{Walonoski2017}. Finally, we validate on two real-world EHR datasets each containing over 500 million clinical events. The AMC-EHR dataset contains longitudinal electronic health records from 647,447 patients at Stanford Medicine in the United States. The second dataset contains two separate cohorts provided by the United Kingdom Clinical Practice Research Datalink (CPRD) \citep{Herrett2015}. The CPRD dataset provides a natural training and testing split, where the training ``Aurum" cohort comprises 815,092 patients drawn from primary-care practices and the externally validated ``Gold"  cohort comprises 53,599 patients from a separate set of practices and EHR software. Full cohort definitions, preprocessing, and split procedures are provided in Appendix \ref{app:experimental_details}. We evaluate forecasting at $\Delta\in\{90,180,365,730,1095\}$ days from the observed anchor. Our primary metric is top-5 clinical-code accuracy, which measures whether the recorded target code appears among the five highest-ranked predictions. This accommodates the possibility that patients have several concurrent conditions, only some of which are diagnosed or recorded at a given time. Top-1 accuracy and macro-AUROC across horizons are reported in Appendix~\ref{app:extended_results}.

\paragraph{Baselines.}
We compare against three families of clinical forecasting methods: first, two decoder-only and encoder--decoder autoregressive transformers \citep{Shmatko2025, Zhang2026Apollo} that directly predict future event tokens; second, GRU-ODE-Bayes \citep{DeBrouwer2019} and Neural CDEs \citep{Kidger2020} which model irregular observations through continuous hidden-state dynamics; and third, four flow-based baselines -- conditional flow matching (CFM) \citep{lipman2023flow}, Metric FM \citep{kapusniak2024metric}, Gaussian-process stream-level flow matching (GP-CFM) \citep{wei2025streamlevel}, and multi-marginal flow matching (MMFM) \citep{rohbeck2025modeling}. The flow-based methods use a common observation representation and vector-field architecture, with an additional geopath network for Metric FM and additional history inputs for \texttt{EHRFlow}. Autoregressive baselines are closely matched in parameter count to the flow-based methods, with only small residual differences due to architectural design. Architectural details, training objectives, and inference procedures are provided in Appendix \ref{app:baselines}.

\subsection{Synthetic trajectory forecasting}

\begin{table}[t]
\centering
\vspace{-15pt}
\caption{Top-5 clinical-code accuracy across synthetic benchmarks,
averaged over forecast horizons. Entries show mean $\pm$ standard
deviation across runs.}
\vspace{2pt}
\label{tab:synthetic_results}
\resizebox{\linewidth}{!}{%
\begin{tabular}{lcccccc}
\toprule
 & Anthracycline & RA on DMARD & CKD & RRMS & Heart failure & Synthea \\
\midrule
Decoder-only & 0.510 $\pm$ 0.017 & 0.821 $\pm$ 0.010 & 0.614 $\pm$ 0.008 & 0.600 $\pm$ 0.023 & 0.405 $\pm$ 0.019 & 0.432 $\pm$ 0.029 \\
Encoder-decoder & 0.507 $\pm$ 0.015 & 0.860 $\pm$ 0.007 & 0.668 $\pm$ 0.016 & 0.593 $\pm$ 0.013 & 0.403 $\pm$ 0.019 & 0.278 $\pm$ 0.121 \\
GRU-ODE-Bayes & 0.546 $\pm$ 0.022 & 0.829 $\pm$ 0.042 & 0.656 $\pm$ 0.022 & 0.524 $\pm$ 0.026 & 0.386 $\pm$ 0.064 & 0.439 $\pm$ 0.052 \\
Neural CDE & 0.411 $\pm$ 0.066 & 0.729 $\pm$ 0.061 & 0.577 $\pm$ 0.069 & 0.491 $\pm$ 0.044 & 0.305 $\pm$ 0.084 & 0.601 $\pm$ 0.027 \\
CFM (adjacent) & 0.574 $\pm$ 0.035 & 0.971 $\pm$ 0.003 & 0.812 $\pm$ 0.014 & 0.781 $\pm$ 0.024 & 0.602 $\pm$ 0.026 & 0.790 $\pm$ 0.014 \\
Metric FM & 0.597 $\pm$ 0.036 & 0.963 $\pm$ 0.004 & 0.789 $\pm$ 0.021 & 0.759 $\pm$ 0.062 & 0.627 $\pm$ 0.053 & 0.766 $\pm$ 0.008 \\
GP-CFM & 0.779 $\pm$ 0.004 & 0.977 $\pm$ 0.001 & 0.912 $\pm$ 0.003 & 0.895 $\pm$ 0.005 & 0.770 $\pm$ 0.005 & \textbf{0.875} $\pm$ 0.004 \\
MMFM & 0.785 $\pm$ 0.003 & 0.980 $\pm$ 0.001 & 0.913 $\pm$ 0.004 & 0.896 $\pm$ 0.003 & 0.767 $\pm$ 0.004 & 0.874 $\pm$ 0.003 \\
\midrule
\texttt{EHRFlow} (GP-CFM) & 0.863 $\pm$ 0.007 & \textbf{0.986} $\pm$ 0.002 & \textbf{0.941} $\pm$ 0.001 & \textbf{0.944} $\pm$ 0.005 & \textbf{0.787} $\pm$ 0.004 & \textbf{0.875} $\pm$ 0.003 \\
\texttt{EHRFlow}  (MMFM) & \textbf{0.865} $\pm$ 0.007 & 0.985 $\pm$ 0.002 & 0.940 $\pm$ 0.004 & 0.942 $\pm$ 0.004 & 0.775 $\pm$ 0.009 & \textbf{0.875} $\pm$ 0.003 \\
\bottomrule
\end{tabular}}
\vspace{-15pt}
\end{table}

We first evaluate across the six synthetic datasets modeling different clinical trajectories. Multi-marginal methods outperform adjacent CFM and Metric FM, with similar results across Gaussian-process and spline interpolants (Table~\ref{tab:synthetic_results}). Our results highlight that history conditioning provides further gains, particularly for cumulative exposure and relapse history. In the aggregate top-5 comparison, both \texttt{EHRFlow} variants outperform all baselines under a paired, stratified sign-flip randomisation test with Benjamini--Hochberg correction ($q<10^{-5}$). On Synthea, top-5 accuracy remains comparable to history-independent multi-marginal models, although history conditioning improves macro-AUROC at longer horizons (see Appendix~\ref{app:extended_results}).

\begin{wraptable}{r}{0.48\textwidth}
\centering
\vspace{-20pt}
\caption{Top-5 clinical-code accuracy on real-world CPRD and AMC-EHR datasets, averaged over forecast horizons.}
\vspace{2pt}
\label{tab:real_ehr}
\resizebox{\linewidth}{!}{%
\begin{tabular}{lcc}
\toprule
 & CPRD & AMC-EHR \\
\midrule
Decoder-only & 0.368 $\pm$ 0.012 & 0.259 $\pm$ 0.008 \\
Encoder-decoder & 0.335 $\pm$ 0.001 & 0.232 $\pm$ 0.037 \\
\midrule
GP-CFM & 0.554 $\pm$ 0.021 & 0.275 $\pm$ 0.004 \\
MMFM & 0.548 $\pm$ 0.014 & 0.279 $\pm$ 0.004 \\
\midrule
\texttt{EHRFlow} (GP-CFM) & \textbf{0.638} $\pm$ 0.009 & 0.283 $\pm$ 0.013 \\
\texttt{EHRFlow} (MMFM) & 0.621 $\pm$ 0.021 & \textbf{0.302} $\pm$ 0.016 \\
\bottomrule
\end{tabular}
}
\vspace{-10pt}
\end{wraptable}

Furthermore, we found that the benefit of history conditioning depends on the evaluation metric. On the Markovian heart-failure dataset, history conditioning provides smaller and less consistent gains than on datasets with explicitly history-dependent dynamics, consistent with history primarily helping recover the underlying state from incomplete observations. We additionally evaluate latent-state recovery against the known simulated latent state at forecasted time-points unobserved in the training data   (Appendix~\ref{app:extended_results}). Predictions at the same time point obtained directly or within longer rollouts show similar accuracy and $93–94\%$ mean top-5 overlap, supporting empirical trajectory consistency (see Table \ref{tab:trajectory_consistency_full}). \texttt{EHRFlow} achieves the lowest latent-state recovery error across synthetic benchmarks, while separate off-grid evaluations demonstrate generalisation to unseen forecast horizons (see Figure \ref{fig:offgrid}). Computational cost is also reduced: synthetic training and evaluation runs take 0.18--0.76 hours with \texttt{EHRFlow}, compared with 10.02--56.94 hours for autoregressive baselines (see Table~\ref{tab:compute}).

\subsection{Real-world trajectory forecasting}
Our results demonstrate that history conditioning also significantly improves horizon-averaged top-5 accuracy in real-world settings ($q < 0.05$, BH-corrected; Table~\ref{tab:real_ehr}). On the CPRD dataset, \texttt{EHRFlow} improves over its history-independent counterparts by 8.4 percentage points with GP-CFM interpolation and 7.3 points with MMFM. On the AMC-EHR dataset, the corresponding gains are 0.8 and 2.3 points. Meanwhile, the comparison with autoregressive methods depends on the metric and forecast horizon. Both \texttt{EHRFlow} variants outperform the autoregressive baselines on horizon-averaged top-5 accuracy. For short horizons, we observe autoregressive models achieve higher macro-AUROC on AMC-EHR, whereas \texttt{EHRFlow} (MMFM) performs best at longer horizons (Appendix~\ref{app:extended_results}). This pattern likely reflects the error accumulation that occurs during autoregressive rollout. The gains on both primary- and secondary-care records indicate that the benefits of history conditioning extend across healthcare settings, including the independent CPRD Gold cohort used for external validation.

\subsection{History-conditioning and optimisation ablations}

We conduct targeted ablations to isolate the contribution of history conditioning, its incorporation, and key optimisation choices (see Table~\ref{tab:ablations}).

\paragraph{Parameter matching.} Because conditioning the vector field on a history representation increases the input dimensionality of the MLP, \texttt{EHRFlow} contains slightly more parameters than the corresponding history-independent GP-CFM and MMFM baselines. We therefore reduce the width of the vector-field MLP to match the baselines' parameter count. The resulting performance is largely unchanged, indicating that our method's improvement does not arise simply from increased model capacity. As a complementary control, we retain the full \texttt{EHRFlow} architecture but replace the history representation with zeros. Performance then returns to that of the history-independent baselines.

\begin{wraptable}{r}{0.52\textwidth}
\centering
\vspace{-22pt}
\caption{\texttt{EHRFlow} ablations: top-5 clinical-code accuracy
averaged across synthetic datasets and forecast horizons.}
\vspace{2pt}
\label{tab:ablations}
\resizebox{\linewidth}{!}{%
\begin{tabular}{lcc}
\toprule
 & \multicolumn{2}{c}{Average top-5 acc.} \\
\cmidrule(lr){2-3}
 & GP-CFM & MMFM \\
\midrule
\texttt{EHRFlow} & \textbf{0.899} $\pm$ 0.001 & \textbf{0.897} $\pm$ 0.002 \\
\texttt{EHRFlow}  matched & 0.897 $\pm$ 0.002 & \textbf{0.897} $\pm$ 0.002 \\
\midrule
\texttt{EHRFlow}, zero context & 0.868 $\pm$ 0.002 & 0.869 $\pm$ 0.001 \\
\texttt{EHRFlow}, end-to-end GRU & 0.873 $\pm$ 0.004 & 0.878 $\pm$ 0.003 \\
\texttt{EHRFlow}, fully end-to-end & 0.876 $\pm$ 0.008 & 0.865 $\pm$ 0.004 \\
\texttt{EHRFlow}, no decoder loss & 0.599 $\pm$ 0.013 & 0.606 $\pm$ 0.011 \\
History as source & 0.865 $\pm$ 0.003 & 0.867 $\pm$ 0.003 \\
History as source + context & 0.864 $\pm$ 0.007 & 0.865 $\pm$ 0.003 \\
No history & 0.868 $\pm$ 0.001 & 0.869 $\pm$ 0.002 \\
\midrule
CFM (adjacent) + history & \multicolumn{2}{c}{0.776 $\pm$ 0.005} \\
CFM (adjacent) & \multicolumn{2}{c}{0.755 $\pm$ 0.007} \\
CFM (MMFM windows) + history & \multicolumn{2}{c}{0.808 $\pm$ 0.007} \\
\bottomrule
\end{tabular}}
\vspace{-10pt}
\end{wraptable}

\paragraph{How history is incorporated.} We next show that using the history representation as the flow source, regardless of whether we additionally condition on history, performs worse than our choice of using the current patient state as the source while conditioning on history. We also show that training the history encoder jointly with the flow objective, instead of pretraining and freezing it, performs worse than using a pretrained history representation. Fully end-to-end optimisation of the latent encoder, decoder, history encoder, and vector field likewise reduces accuracy, supporting the use of pretrained representations under the evaluated optimisation settings. We also found that removing the auxiliary decoder loss, thereby freezing it, substantially reduces accuracy, highlighting the importance of observation-space supervision during flow training. Finally, adding history conditioning to adjacent CFM also improves forecasting, but remains below the multi-marginal models. These ablations indicate that the gains of \texttt{EHRFlow} arise from both explicit historical conditioning and multi-observation trajectory supervision. Dataset-level results are provided in Appendix \ref{app:extended_results}. We additionally vary the history-window length to assess sensitivity to the amount of preceding context.

\subsection{Guided intervention modelling}
As an additional controlled experiment, we test whether \texttt{EHRFlow} captures intervention-dependent patient dynamics in a simulation with known ground truth.

\textbf{Evaluation protocol.}
We construct a Synthea hypertension experiment in which the same patient population is simulated twice from identical initial random states, where one population has antihypertensive prescriptions suppressed. Treatment is defined as antihypertensive initiation and the outcome as systolic blood pressure, with simulator-defined treatment effect as $\theta=-20.48$ mmHg. We condition \texttt{EHRFlow} on treatment assignment using classifier-free guidance \citep{Ho2022} and evaluate both matched trajectories, where both treatment conditions are available for each patient, and a split-arm setting, where each patient contributes to only one condition. The guidance weight is selected using validation outcome-prediction error without access to $\theta$.

\begin{wraptable}{r}{0.48\textwidth}
\centering
\vspace{-10pt}
\caption{Recovery of the simulator-defined hypertension effect with CFG.}
\vspace{2pt}
\label{tab:treatment_arms}
\resizebox{\linewidth}{!}{%
\begin{tabular}{lccc}
\toprule
Design & Observed & Guided & Error \\
\midrule
\multicolumn{4}{l}{\textbf{EHRFlow (GP-CFM)}} \\
Paired    & $-23.20 \pm 0.59$ & $-19.42 \pm 2.33$ & $1.06 \pm 2.33$ \\
Split-arm & $-22.58 \pm 0.59$ & $-18.41 \pm 1.62$ & $2.07 \pm 1.62$ \\
\midrule
\multicolumn{4}{l}{\textbf{EHRFlow (MMFM)}} \\
Paired    & $-23.20 \pm 0.59$ & $-18.83 \pm 2.27$ & $1.65 \pm 2.27$ \\
Split-arm & $-22.58 \pm 0.59$ & $-20.41 \pm 1.72$ & $0.07 \pm 1.72$ \\
\bottomrule
\end{tabular}%
}
\vspace{-15pt}
\end{wraptable}
\textbf{Results.}
Guided \texttt{EHRFlow} approximates the simulator-defined effect with an average error of 0.07--2.07\,mmHg across both models and evaluation settings, equivalent to a relative error of 0.3\%--10.1\% (Table~\ref{tab:treatment_arms}). These results suggest that the learned latent trajectories can preserve clinically meaningful changes in individual physiological variables, as demonstrated here for blood pressure under a simulated intervention.

\section{Discussion}

Generative models of patient trajectories have attracted growing interest \citep{liu2025conditional}, with language-inspired transformers demonstrating success in forecasting clinical events \citep{Shmatko2025, Waxler2025}. However, autoregressive forecasting typically requires sequential rollout and represents patient evolution through discrete observations. Encoder--decoder architectures such as APOLLO offer alternative ways to encode clinical history and predict future events, but inherently model trajectories as a sequence of discrete events \citep{Zhang2026Apollo}. Continuous-time models instead represent evolving latent states, although numerical integration during training can make learning computationally demanding. Flow matching offers a scalable route to learning these dynamics through simulation-free vector-field training.

Applying flow matching to longitudinal EHRs requires several adaptations. Pairwise conditional flow matching supervises paths between two observations, whereas multi-marginal methods incorporate multiple observations that guide the trajectory \citep{rohbeck2025modeling,wei2025streamlevel}. However, multi-observation supervision does not itself provide inference-time access to preceding patient history. Clinical forecasting also requires suitable latent representations and accommodation of irregular observations and varying forecast durations. \texttt{EHRFlow} combines pretrained observation and history encoders with history-conditioned dynamics and duration-dependent normalised flow time. Jointly fine-tuning the vector field and decoder with a single-step linear endpoint surrogate provides observation-space supervision without numerical integration during training.

The results support the value of combining continuous latent dynamics with historical context for individual patient forecasting. Historical context helps distinguish patients with similar current records but different subsequent trajectories, with benefits extending to large-scale primary- and secondary-care datasets. Furthermore, our results on synthetic state recovery and forecasting at unseen horizons demonstrate the quality of our model's learned patient representation. Our method's similar performance across rollout durations further supports forecast consistency, while similar performance across interpolant families suggests that the findings are not specific to path construction.

\paragraph{Limitations.}
Continuous latent modelling does not guarantee recovery of biological dynamics, as EHRs also reflect coding practices, missingness, and informative observation times driven by healthcare utilisation that our model does not explicitly capture. Our evaluation is limited to continuous and discrete clinical variables, although the encoder-based formulation supports incorporating histopathology, imaging, and genomic data; the effectiveness of these extensions remains to be demonstrated. Predictive gains also depend on the metric and horizon, with autoregressive models remaining competitive on macro-AUROC at shorter horizons. Our staged training procedure outperforms fully end-to-end optimisation under the evaluated settings, motivating future work on objectives that more effectively co-adapt latent representations and dynamics. Finally, the intervention experiment is a controlled proof of concept rather than evidence of causal treatment-effect estimation from observational data. Broader external validation, uncertainty assessment, and prospective evaluation remain necessary before clinical use.

\paragraph{Conclusion and clinical perspective.}
\texttt{EHRFlow} enables continuous, history-conditioned forecasting from population-scale clinical records, potentially supporting earlier intervention and personalised care. Future work should examine whether the evolving latent embeddings can also serve as reusable representations for downstream clinical tasks \citep{Zhang2026Apollo}. More broadly, this formulation offers a step towards patient digital twins that could help explore treatment strategies and optimise care, subject to further clinical validation.

\section{Acknowledgements}

The Cambridge Centre for AI in Medicine (CCAIM) receives funding from
GSK, Boehringer-Ingelheim, AstraZeneca, Sanofi, Takeda and Quantum Black,
AI by McKinsey. This study is based in part on data from the Clinical Practice Research Datalink
(CPRD) obtained under licence from the UK Medicines and Healthcare products
Regulatory Agency (MHRA). This work uses data provided by patients and collected
by the NHS as part of their care and support. The interpretation and conclusions
contained in this study are those of the authors alone. Access to the CPRD data was supported by the NIHR and the BHF. This research used data or services provided by STARR, ``STAnford medicine
Research data Repository,'' a clinical data warehouse containing live Epic data
from Stanford Health Care, the Stanford Children's Hospital, the University
Healthcare Alliance and Packard Children's Health Alliance clinics and other
auxiliary data from hospital applications such as radiology PACS. The STARR
platform is developed and operated by the Stanford Medicine Research Technology
team and is made possible by funding from the Stanford School of Medicine
Dean's Office.
\bibliography{iclr2027_conference}
\bibliographystyle{iclr2027_conference}

\newpage

\appendix

\AppendixBanner

\section{Extended related work}

\paragraph{Patient representations and generative latent dynamics.}
Beyond the clinical sequence models discussed in the main text, pretrained models learn reusable representations from structured electronic health records for transfer to downstream clinical tasks \citep{Steinberg2020CLMBR,Guo2024,Pang2024cehr}. Earlier deep generative state-space models learned latent transition and observation distributions from longitudinal data, including applications to patient trajectories and treatment counterfactuals \citep{Krishnan2015,Krishnan2017}. 

\paragraph{Synthetic longitudinal health records.}
A complementary literature develops generative models for synthetic health data. Early approaches such as medGAN generate discrete patient records using adversarial training \citep{Choi2017medgan}, while recurrent conditional GANs generate multivariate medical time series \citep{Esteban2017}. Subsequent methods model longitudinal EHR structure, including PromptEHR \citep{wang-sun-2022-promptehr} and the hierarchical autoregressive model HALO \citep{Theodorou2023}. Other approaches jointly generate continuous measurements and discrete clinical variables \citep{Li2023MixedEHR}, while PatientFlow applies flow matching to generate synthetic longitudinal clinical records \citep{Branco2026}. 

\paragraph{Diffusion-based time-series modeling.}
Diffusion models provide another framework for generative modeling of temporal data. TimeGrad combines autoregressive forecasting with conditional denoising diffusion \citep{Rasul2021TimeGrad}, while CSDI conditions score-based diffusion on observed values for probabilistic imputation and also supports interpolation and forecasting \citep{Tashiro2021CSDI}. Non-autoregressive conditional diffusion models generate future sequences jointly, avoiding sequential sampling across forecast time steps \citep{Shen2023TimeDiff}. These approaches offer alternatives for modeling distributions over temporal observations, whereas \texttt{EHRFlow} learns a vector field describing latent patient evolution over the forecast interval.

\paragraph{Biological trajectories and conditional transport.}
Related trajectory-inference problems arise in biology, where destructive measurements often prevent repeated observation of the same cell. Waddington-OT infers couplings between cell populations sampled at different times \citep{Schiebinger2019}, while TrajectoryNet learns continuous transport dynamics from such snapshots \citep{Tong2020TrajectoryNet}. Manifold-based approaches incorporate data geometry into trajectory reconstruction \citep{huguet2022manifold}, and CellBRIDGE incorporates interaction information when aligning cellular states across time \citep{estevez2026cellbridge}. Meta Flow Matching conditions on population embeddings to generalise across initial distributions \citep{atanackovic2025meta}, while related work combines score and flow matching to learn stochastic dynamics through Schr\"odinger bridges from unpaired samples \citep{Tong2024schrodinger}. Longitudinal EHRs provide within-patient correspondence across observations, allowing us to construct training paths directly and condition forecasts on individual history.

\paragraph{Treatment-conditioned generation and counterfactual prediction.}
Classifier-free guidance combines conditional and unconditional model predictions to control generation without a separate classifier \citep{Ho2022}. We use this mechanism to generate treatment-conditioned trajectories in a controlled simulation. A distinct literature targets counterfactual outcome prediction from observational data: recurrent marginal structural networks model treatment assignment and outcomes over time \citep{Lim2018fore}, while Counterfactual Recurrent Networks \citep{Bica2020CRN} and the Causal Transformer \citep{Melnychuk2022CT} learn representations designed to address time-varying confounding. These methods rely on causal identification assumptions. Our intervention experiment instead evaluates recovery of a known simulator-defined treatment effect; conditional guidance alone does not establish causal identification from observational EHRs.

\section{Theoretical motivation for history conditioning}
\label{app:theory}

We establish that history conditioning cannot increase the optimal population flow-matching loss,
and give a simple example separating population marginal matching from individual trajectory forecasting.
\subsection{History conditioning reduces the optimal flow-matching loss}
\begin{theorem}[Optimality of history conditioning]
Let
\[
X=(Z_s,s,\lambda),
\]
let $H$ denote the history representation, and let $U\in\mathbb{R}^{d_z}$ be the target velocity with $\mathbb{E}\|U\|_2^2<\infty$. Define
\[
R_0^*
=
\inf_v \mathbb{E}\|U-v(X)\|_2^2,
\qquad
R_H^*
=
\inf_w \mathbb{E}\|U-w(X,H)\|_2^2.
\]
Then
\[
R_0^*-R_H^*
=
\mathbb{E}
\left\|
\mathbb{E}[U\mid X,H]
-
\mathbb{E}[U\mid X]
\right\|_2^2
\geq 0.
\]
Equality holds if and only if
\[
\mathbb{E}[U\mid X,H]
=
\mathbb{E}[U\mid X]
\quad\text{a.s.}
\]
\end{theorem}

\begin{proof}
Squared-error risk is minimised by the conditional mean. Let
\[
m_0=\mathbb{E}[U\mid X],
\qquad
m_H=\mathbb{E}[U\mid X,H].
\]
Then
\[
R_0^*
=
\mathbb{E}\|U-m_0\|_2^2
=
\mathbb{E}\|(U-m_H)+(m_H-m_0)\|_2^2.
\]
Since
\[
\mathbb{E}[U-m_H\mid X,H]=0,
\]
the cross term vanishes, giving
\[
R_0^*
=
R_H^*
+
\mathbb{E}\|m_H-m_0\|_2^2.
\]
\end{proof}

Thus, the benefit of history is determined by the squared change it induces in the conditional mean velocity. If the expected dynamics are already determined by the current state and time, history provides no improvement.

\subsection{Population marginals do not determine individual trajectories}
 We next give a simple Gaussian example showing that the same distinction also appears after integrating the learned field: two fields can generate the same population distribution while producing different forecasts for individual trajectories.

Consider
\[
Z_t=A+tB,
\qquad
A,B\overset{\mathrm{iid}}{\sim}\mathcal{N}(0,1),
\qquad
t\in[-1,1].
\]
Here $A$ determines the state at the forecast anchor and $B$ determines the patient-specific rate of progression. The anchor is
\[
Z_0=A.
\]
Suppose the previous state is observed with noise,
\[
Y_{-1}=Z_{-1}+\eta,
\qquad
\eta\sim\mathcal{N}(0,\sigma^2),
\]
independently of $(A,B)$. We define the history representation
\[
H=Z_0-Y_{-1}=B-\eta.
\]
Its reliability is
\[
\rho^2
=
\operatorname{Corr}(H,B)^2
=
\frac{1}{1+\sigma^2},
\qquad
\kappa=1-\rho^2.
\]
Thus, $\rho^2$ increases as the history becomes more informative about the patient's progression rate, with $\rho=1$ corresponding to noiseless history.

During training the model uses the path
\[
Z_s=A+sB,
\qquad
U=B,
\qquad
s\sim\mathcal{U}(0,1).
\]

\begin{theorem}[Population marginals versus individual trajectories]
For every $\rho\in(0,1]$, the following hold:

\begin{enumerate}[leftmargin=*,itemsep=0.4ex,parsep=-0.1ex,topsep=0.25ex]
    \item \textbf{Optimal fields.}
    The population-optimal history-independent and history-conditioned fields are
    \[
    v^*(z,s)
    =
    \frac{s}{1+s^2}z,
    \qquad
    w^*(z,s,h)
    =
    \rho^2 h
    +
    \frac{s\kappa}{1+\kappa s^2}
    \left(z-s\rho^2 h\right).
    \]

    \item \textbf{Flow-matching risks.}
    The corresponding optimal risks are
    \[
    R_0^*
    =
    \frac{\pi}{4},
    \qquad
    R_H^*
    =
    \sqrt{\kappa}\arctan\sqrt{\kappa}.
    \]
    Hence $R_H^*<R_0^*$ for $\rho>0$, and $R_H^*=0$ when $\rho=1$.

    \item \textbf{Integrated flows.}
    Starting from $Z_0=A$, the resulting trajectories are
    \[
    \widehat Z_s^{\,0}
    =
    A\sqrt{1+s^2},
    \qquad
    \widehat Z_s^{\,H}
    =
    s\rho^2H
    +
    A\sqrt{1+\kappa s^2}.
    \]
    Both recover the correct population marginals:
    \[
    \widehat Z_s^{\,0}
    \overset{d}{=}
    \widehat Z_s^{\,H}
    \overset{d}{=}
    Z_s
    \sim
    \mathcal{N}(0,1+s^2),
    \]
    and the history-conditioned flow additionally satisfies
    \[
    \widehat Z_s^{\,H}\mid H
    \overset{d}{=}
    Z_s\mid H.
    \]

    \item \textbf{Individual forecast error.}
    For every $s\in(0,1]$,
    \[
    \mathbb{E}
    \left[
    \left(\widehat Z_s^{\,0}-Z_s\right)^2
    \right]
    =
    \left(\sqrt{1+s^2}-1\right)^2+s^2,
    \]
    whereas
    \[
    \mathbb{E}
    \left[
    \left(\widehat Z_s^{\,H}-Z_s\right)^2
    \right]
    =
    \left(\sqrt{1+\kappa s^2}-1\right)^2+\kappa s^2.
    \]
    The history-conditioned error is strictly smaller for $\rho>0$,
    decreases with $\rho^2$, and vanishes when $\rho=1$.
\end{enumerate}
\end{theorem}

\begin{proof}
Since $(B,H)$ is jointly Gaussian,
\[
B\mid H=h
\sim
\mathcal{N}(\rho^2h,\kappa).
\]
Conditionally on $H=h$,
\[
\mathbb{E}[Z_s\mid h]=s\rho^2h,
\qquad
\operatorname{Var}(Z_s\mid h)=1+\kappa s^2,
\qquad
\operatorname{Cov}(B,Z_s\mid h)=s\kappa.
\]
Gaussian conditioning therefore gives
\[
\mathbb{E}[B\mid Z_s=z,H=h]
=
\rho^2h
+
\frac{s\kappa}{1+\kappa s^2}
\left(z-s\rho^2h\right),
\]
which is $w^*$. Without conditioning on history,
\[
\mathbb{E}[B\mid Z_s=z]
=
\frac{s}{1+s^2}z,
\]
which gives $v^*$.

The corresponding conditional variances are
\[
\operatorname{Var}(B\mid Z_s)
=
\frac{1}{1+s^2},
\qquad
\operatorname{Var}(B\mid Z_s,H)
=
\frac{\kappa}{1+\kappa s^2}.
\]
Integrating over $s\in[0,1]$ yields
\[
R_0^*
=
\int_0^1\frac{1}{1+s^2}\,ds
=
\frac{\pi}{4},
\]
and
\[
R_H^*
=
\int_0^1\frac{\kappa}{1+\kappa s^2}\,ds
=
\sqrt{\kappa}\arctan\sqrt{\kappa}.
\]

Integrating the optimal fields from $\widehat Z_0=A$ gives
\[
\widehat Z_s^{\,0}
=
A\sqrt{1+s^2},
\qquad
\widehat Z_s^{\,H}
=
s\rho^2H+A\sqrt{1+\kappa s^2}.
\]
Since $A$ is independent of $H$,
\[
\widehat Z_s^{\,H}\mid H=h
\sim
\mathcal{N}
\left(
s\rho^2h,
1+\kappa s^2
\right),
\]
which is exactly the conditional distribution of $Z_s\mid H=h$. Marginalising over $H$ gives
\[
\widehat Z_s^{\,H}
\sim
\mathcal{N}(0,1+s^2),
\]
and the same population marginal is obtained by $\widehat Z_s^{\,0}$.

Finally, write
\[
B=\rho^2H+R,
\qquad
R\sim\mathcal{N}(0,\kappa),
\]
with $R$ independent of $(A,H)$. Then
\[
\widehat Z_s^{\,H}-Z_s
=
A\left(\sqrt{1+\kappa s^2}-1\right)-sR,
\]
while
\[
\widehat Z_s^{\,0}-Z_s
=
A\left(\sqrt{1+s^2}-1\right)-sB.
\]
Independence of the two terms in each expression gives the stated mean-squared errors. The history-conditioned error is strictly increasing in $\kappa=1-\rho^2$, and therefore decreases as the history becomes more reliable. For $\rho=1$, we have $\kappa=0$ and $H=B$, so
\[
\widehat Z_s^{\,H}=A+sB=Z_s.
\]
\end{proof}

Figure~\ref{fig:theorem-2-exp} evaluates Theorem~2 empirically. We train both fields as small MLPs by flow-matching regression on samples from the linear-Gaussian model, integrate them from $Z_0 = A$, and evaluate on held-out samples for $\rho \in \{0.1, \dots, 1\}$. The learned flows match the closed forms in all three respects: forecast error, optimal risk, and marginal and conditional variance.

\begin{figure}[h]
    \centering
    \includegraphics[width=\linewidth]{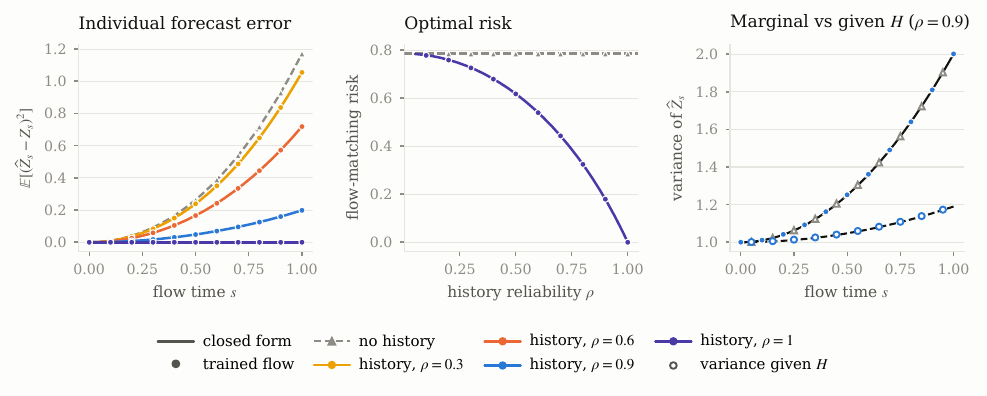}
    \caption{ \textbf{Left:} individual forecast error against flow time. \textbf{Centre:} flow-matching risk against history reliability $\rho$. \textbf{Right:} variance of the generated state at $\rho = 0.9$; both flows recover the marginal $1+s^2$, only the history-conditioned flow recovers the variance given $H$, $1+\kappa s^2$.}
    \label{fig:theorem-2-exp}
\end{figure}

\newpage

\section{Extended results}
\label{app:extended_results}

\subsection{The choice of the interpolant}

Across both the synthetic and real-world datasets, we observed that the choice of multi-marginal interpolant had a substantially smaller effect on predictive performance than history conditioning. This suggests that, provided the interpolant captures the observed marginals, much of the variation between particular interpolation schemes may be attenuated when the corresponding vector field is learned by the neural network.

We illustrate this effect in Figure~\ref{fig:interpolant_equivalence}. For a fixed set of observed marginals, multiple valid interpolation paths exist that pass through the same knots. Although different interpolants induce substantially different conditional velocity
targets, flow matching learns a marginal vector field by regressing over these conditional
targets \citep{lipman2023flow,albergo2023building}. This regression can therefore attenuate
interpolant-specific variability, providing a possible explanation for the substantially
smaller differences observed between the fitted fields in
Figure~\ref{fig:interpolant_equivalence}. 

\begin{figure}[!htb]
    \centering
    \includegraphics[width=\linewidth]{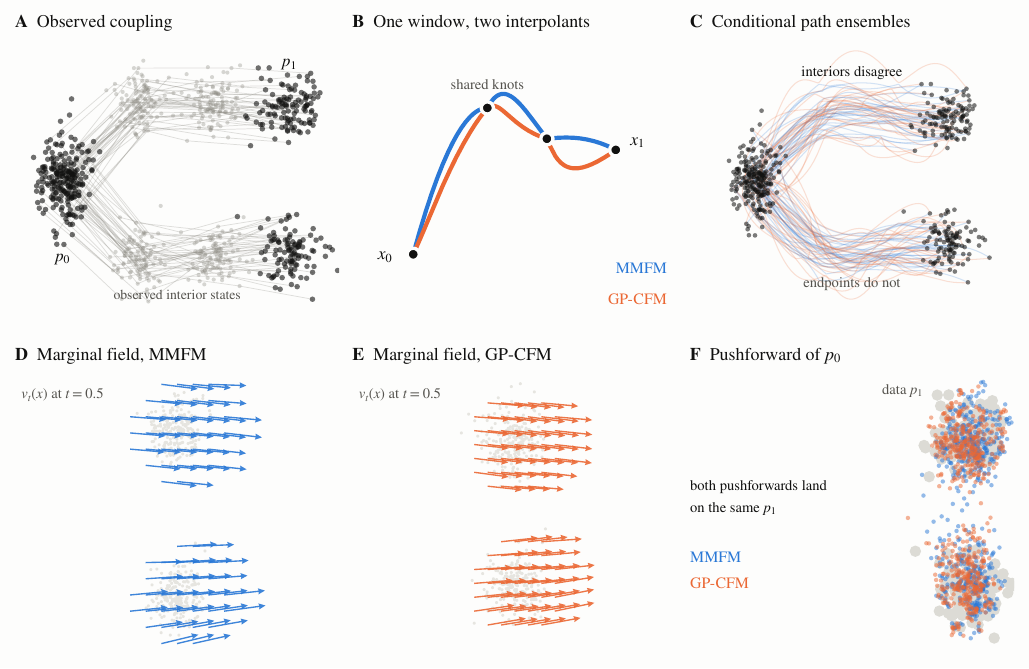}
    \caption{Smooth vector-field paths through the same knots on synthetic data.}
    \label{fig:interpolant_equivalence} 
\end{figure}
We examine this phenomenon more systematically in Figure~\ref{fig:interpolant_smoothing} using a two-dimensional synthetic example with six different knot-interpolating families. The conditional paths, and consequently their velocity targets, differ substantially across interpolants. However, these differences are reduced after fitting the same MLP architecture to the corresponding targets, consistent with the network averaging over variation in the conditional paths. Differences are reduced further after integrating the fitted vector fields: while small discrepancies remain at $t=0.5$, the final pushforward distributions are nearly indistinguishable. Across ten seeds, pairwise Wasserstein-1 distances between the predicted terminal distributions are only $0.02$--$0.04$. 
\begin{figure}[!htb]
    \centering
    \includegraphics[width=\linewidth]{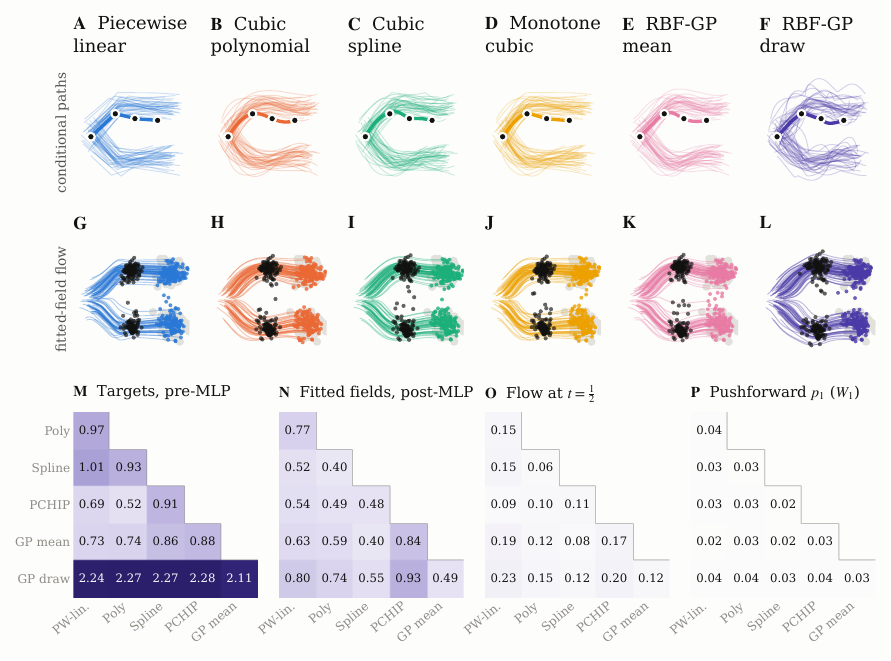}
    \caption{Six knot-interpolating families fitted with the same network architecture (10 seeds). Differences between conditional paths are reduced in the fitted flows, with pairwise endpoint Wasserstein-1 distances of 0.02--0.04.}
    \label{fig:interpolant_smoothing} 
\end{figure}

\newpage

\subsection{Extended forecasting metrics}

We extend the horizon-averaged results in the main text by reporting top-5 accuracy, top-1 accuracy, and macro-AUROC separately at each forecast horizon. These metrics assess complementary aspects of prediction: top-5 accuracy measures whether the recorded code appears among the five highest-ranked predictions, top-1 accuracy requires it to be the highest-ranked prediction, and macro-AUROC assesses discrimination across evaluated code classes.

\paragraph{Synthetic benchmarks.}
Figures~\ref{fig:horizon_top_5_syn}, \ref{fig:horizon_top_1_syn} and \ref{fig:horizon_auroc_syn} report results across the five controlled synthetic datasets and Synthea. Multi-marginal methods generally outperform adjacent CFM and Metric FM, while history conditioning provides additional gains that depend on the data-generating mechanism. The clearest benefits occur in Anthracycline and RRMS, where preceding observations provide information about cumulative exposure and relapse history. Improvements across both top-1 accuracy and macro-AUROC indicate that these benefits extend beyond the top-5 metric used in the main text.

The additional metrics also reveal settings in which history conditioning provides limited benefit. On Heart failure, the multi-marginal variants have similar top-1 accuracy, as expected for the Markovian disease mechanism, despite modest differences in top-5 accuracy. On Synthea, history conditioning yields little improvement in top-5 accuracy but notably improves macro-AUROC at longer horizons. Performance generally declines with increasing forecast duration, although this pattern is not universal across datasets and metrics.

\begin{figure}[h]
    \centering
    \includegraphics[width=\linewidth]{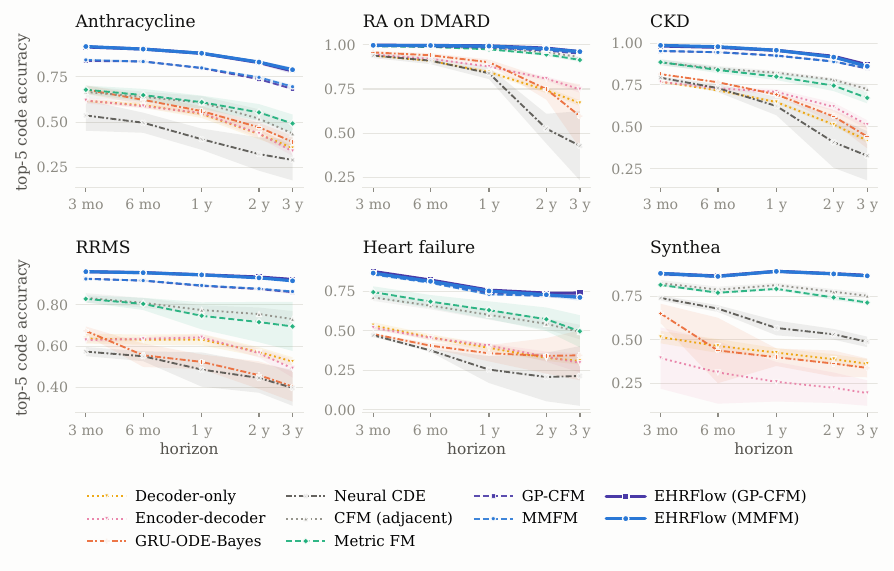}
    \caption{Clinical-code top-5 accuracy by forecast horizon across the five  synthetic datasets and Synthea. Lines show the mean across random seeds; bands indicate $\pm$ one standard deviation.}
    \label{fig:horizon_top_5_syn} 
\end{figure}

\begin{figure}[h]
    \centering
    \includegraphics[width=\linewidth]{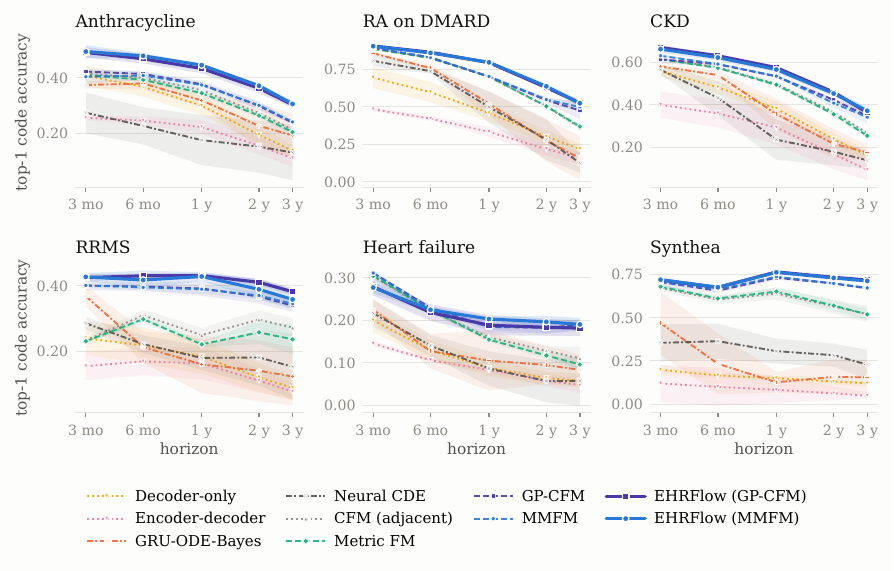}
    \caption{Clinical-code top-1 accuracy by forecast horizon across the five  synthetic datasets and Synthea. Lines show the mean across random seeds; bands indicate $\pm$ one standard deviation.}
    \label{fig:horizon_top_1_syn} 
\end{figure}

\newpage

\begin{figure}[h]
    \centering
    \includegraphics[width=\linewidth]{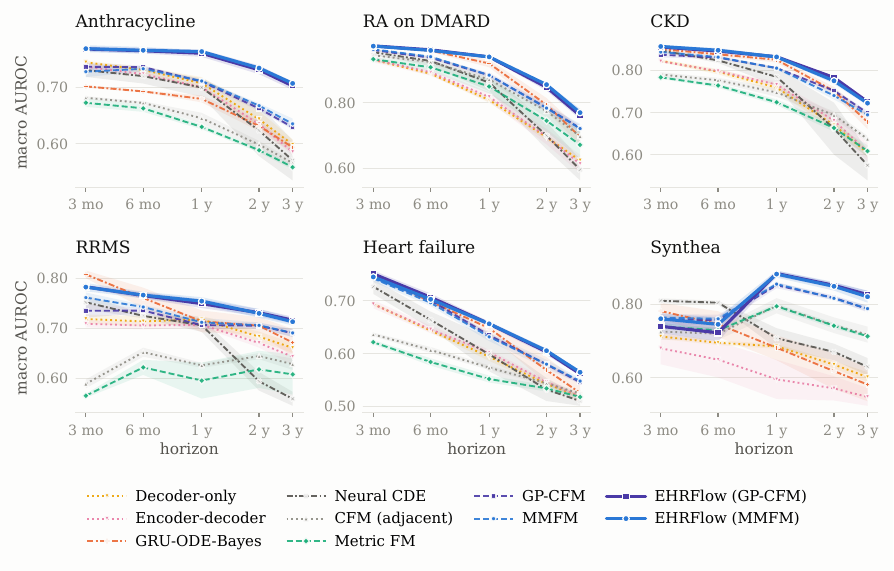}
    \caption{Clinical-code macro-AUROC by forecast horizon across the five  synthetic datasets and Synthea. Lines show the mean across random seeds; bands indicate $\pm$ one standard deviation.}
    \label{fig:horizon_auroc_syn} 
\end{figure}

\begin{figure}[h]
    \centering
    \includegraphics[width=\linewidth]{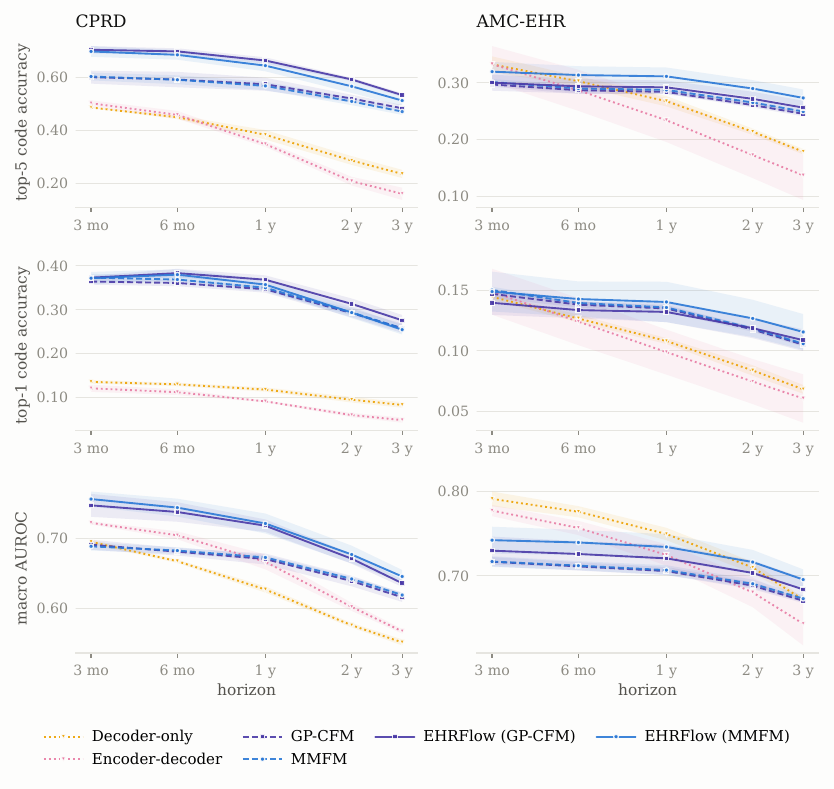}
    \caption{Clinical-code forecasting performance by horizon on CPRD (external Gold cohort) and AMC-EHR. Rows show top-5 accuracy, top-1 accuracy, and macro-AUROC. Lines show the mean across random seeds; shaded bands indicate $\pm$ one standard deviation.}
    \label{fig:horizon_real} 
    
\end{figure}

\paragraph{Real-world datasets.}
Figure~\ref{fig:horizon_real} reports all three metrics on CPRD and AMC-EHR. On CPRD, history conditioning produces clear improvements over the corresponding history-independent flow models in top-5 accuracy and macro-AUROC, while the differences in top-1 accuracy are smaller. The autoregressive baselines deteriorate more strongly with increasing horizon, particularly for top-5 accuracy. On AMC-EHR, the relative ranking depends more on the metric: autoregressive models remain competitive at short horizons and achieve higher short-horizon macro-AUROC, whereas \texttt{EHRFlow} (MMFM) performs best at longer horizons.

\subsection{Evidence for continuous latent dynamics}

Accurate predictions at the training horizons alone do not establish that a model captures useful latent dynamics between them. We therefore examine three complementary properties: recovery of the underlying simulated state, agreement between forecasts of the same time point obtained using different rollout durations, and predictive performance at horizons outside the training grid.

\paragraph{Recovery of the simulated latent state.}
Table~\ref{tab:latent_recovery} evaluates whether forecasted latent representations retain information about the simulator's underlying disease-severity state. We read out severity using a linear map fitted on encoded anchor states and measure mean-squared error against the ground-truth severity at the forecast horizon. Results are averaged across the 90-, 365-, and 1,095-day horizons and summarised by the median and standard deviation across seeds.

History conditioning reduces recovery error relative to the corresponding history-independent multi-marginal models across all five controlled datasets. For example, on RRMS, error decreases from 2.23 to 1.58 for GP-CFM and from 2.35 to 1.67 for MMFM. On CKD, the corresponding reductions are from 0.99 to 0.78 and from 1.01 to 0.73. These results suggest that the forecasting improvements are accompanied by better recovery of the simulated state, rather than being confined to clinical-code prediction.

\paragraph{Interpolation and extrapolation across forecast horizons.}
Table~\ref{tab:offgrid_accuracy} and Figure~\ref{fig:offgrid} evaluate forecasts at horizons not included in the training grid. We distinguish interior horizons of 270, 545, and 912 days from horizons outside the trained range: 45 days and 1,460 days. The latter horizons test predictions at shorter and longer durations than the nominal training horizons, respectively. Figure~\ref{fig:offgrid} places these evaluations alongside the trained horizons to show how performance changes with forecast duration.

At interior horizons, history conditioning retains clear benefits on Anthracycline, CKD, and RRMS. For example, Anthracycline top-5 accuracy increases from 0.767 to 0.851 for GP-CFM and from 0.774 to 0.853 for MMFM. These gains also persist in the aggregate evaluation outside the trained range. Improvements are less consistent on Heart failure and Synthea, and the history-conditioned variants do not outperform their counterparts in every setting. 

\paragraph{Agreement across rollout durations.}
Because the vector field is conditioned on the requested duration, forecasts of the same calendar time do not necessarily agree when obtained from rollouts of different lengths. Table~\ref{tab:trajectory_consistency_full} examines this directly. For each nested pair $h < H$, we compare a forecast integrated to day $h$ with the state at day $h$ along a rollout targeting day $H$, using the same anchor and historical context. We report the top-5 accuracy of each forecast and the fraction of codes shared between their top-5 prediction sets.

Across the evaluated datasets and horizon pairs, mean top-5 overlap is 93.9\% for \texttt{EHRFlow} (GP-CFM) and 93.1\% for \texttt{EHRFlow} (MMFM). Direct and nested forecasts also have similar mean accuracy: 0.920 versus 0.921 for GP-CFM and 0.919 versus 0.921 for MMFM. Agreement varies across datasets and horizon pairs, with lower overlap for some longer rollouts. 
\begin{table}[!htbp]
\centering
\caption{Median and s.d. of latent-state recovery on the synthetic benchmarks across seeds.}
\vspace{2pt}
\label{tab:latent_recovery}
\resizebox{\textwidth}{!}{%
\begin{tabular}{lccccc}
\toprule
 & Heart failure & RA on DMARD & CKD & RRMS & Anthracycline \\
\midrule
CFM (adjacent) & 1.48 $\pm$ 0.05 & 0.36 $\pm$ 0.04 & 1.14 $\pm$ 0.01 & 1.96 $\pm$ 0.06 & 1.07 $\pm$ 0.06 \\
Metric FM & 1.41 $\pm$ 0.08 & 0.34 $\pm$ 0.03 & 1.10 $\pm$ 0.06 & 2.06 $\pm$ 0.04 & 1.35 $\pm$ 0.23 \\
GP-CFM & 1.24 $\pm$ 0.06 & 0.26 $\pm$ 0.00 & 0.99 $\pm$ 0.05 & 2.23 $\pm$ 0.12 & 1.07 $\pm$ 0.05 \\
MMFM & 1.34 $\pm$ 0.02 & 0.27 $\pm$ 0.01 & 1.01 $\pm$ 0.06 & 2.35 $\pm$ 0.11 & 0.98 $\pm$ 0.01 \\
\midrule
\texttt{EHRFlow} (GP-CFM) & \textbf{1.18} $\pm$ 0.12 & \textbf{0.23} $\pm$ 0.00 & 0.78 $\pm$ 0.03 & \textbf{1.58} $\pm$ 0.06 & \textbf{0.81} $\pm$ 0.01 \\
\texttt{EHRFlow} (MMFM) & 1.24 $\pm$ 0.12 & 0.25 $\pm$ 0.01 & \textbf{0.73} $\pm$ 0.12 & 1.67 $\pm$ 0.08 & 0.94 $\pm$ 0.05 \\
\bottomrule
\end{tabular}}
\end{table}

\begin{table}[htbp]
\centering
\caption{Top-5 clinical-code accuracy at the off-grid horizons, averaged over the horizons of each block (mean $\pm$ std over seeds).}
\label{tab:offgrid_accuracy}
\resizebox{\textwidth}{!}{%
\begin{tabular}{lcccccc}
\toprule
 & Heart failure & RA on DMARD & CKD & RRMS & Anthracycline & Synthea \\
\midrule
\multicolumn{7}{l}{\textit{Interior off-grid horizons (270, 545, 912 d)}} \\
GRU-ODE-Bayes & 0.352 $\pm$ 0.091 & 0.813 $\pm$ 0.043 & 0.611 $\pm$ 0.025 & 0.495 $\pm$ 0.013 & 0.518 $\pm$ 0.020 & 0.302 $\pm$ 0.031 \\
Neural CDE & 0.267 $\pm$ 0.119 & 0.664 $\pm$ 0.066 & 0.492 $\pm$ 0.100 & 0.438 $\pm$ 0.055 & 0.386 $\pm$ 0.098 & 0.562 $\pm$ 0.032 \\
CFM (adjacent) & 0.573 $\pm$ 0.030 & 0.967 $\pm$ 0.003 & 0.797 $\pm$ 0.012 & 0.769 $\pm$ 0.029 & 0.557 $\pm$ 0.040 & 0.705 $\pm$ 0.012 \\
Metric FM & 0.599 $\pm$ 0.063 & 0.957 $\pm$ 0.006 & 0.769 $\pm$ 0.024 & 0.743 $\pm$ 0.079 & 0.583 $\pm$ 0.042 & 0.664 $\pm$ 0.018 \\
GP-CFM & 0.738 $\pm$ 0.006 & 0.974 $\pm$ 0.001 & 0.906 $\pm$ 0.004 & 0.887 $\pm$ 0.005 & 0.767 $\pm$ 0.003 & 0.813 $\pm$ 0.006 \\
MMFM & 0.735 $\pm$ 0.005 & 0.977 $\pm$ 0.002 & 0.905 $\pm$ 0.004 & 0.889 $\pm$ 0.004 & 0.774 $\pm$ 0.004 & \textbf{0.814} $\pm$ 0.003 \\
\midrule
\texttt{EHRFlow} (GP-CFM) & \textbf{0.756} $\pm$ 0.006 & \textbf{0.985} $\pm$ 0.002 & \textbf{0.934} $\pm$ 0.003 & \textbf{0.939} $\pm$ 0.005 & 0.851 $\pm$ 0.009 & 0.810 $\pm$ 0.005 \\
\texttt{EHRFlow} (MMFM) & 0.746 $\pm$ 0.008 & 0.984 $\pm$ 0.002 & 0.933 $\pm$ 0.005 & 0.937 $\pm$ 0.003 & \textbf{0.853} $\pm$ 0.007 & 0.806 $\pm$ 0.006 \\
\midrule
\multicolumn{7}{l}{\textit{Outside the trained range (45, 1460 d)}} \\
GRU-ODE-Bayes & 0.476 $\pm$ 0.100 & 0.708 $\pm$ 0.157 & 0.627 $\pm$ 0.031 & 0.579 $\pm$ 0.020 & 0.548 $\pm$ 0.003 & 0.521 $\pm$ 0.026 \\
Neural CDE & 0.407 $\pm$ 0.105 & 0.677 $\pm$ 0.131 & 0.585 $\pm$ 0.087 & 0.478 $\pm$ 0.047 & 0.414 $\pm$ 0.094 & 0.603 $\pm$ 0.039 \\
CFM (adjacent) & 0.586 $\pm$ 0.029 & 0.947 $\pm$ 0.008 & 0.772 $\pm$ 0.025 & 0.774 $\pm$ 0.027 & 0.534 $\pm$ 0.025 & 0.781 $\pm$ 0.010 \\
Metric FM & 0.600 $\pm$ 0.056 & 0.937 $\pm$ 0.009 & 0.730 $\pm$ 0.023 & 0.754 $\pm$ 0.070 & 0.578 $\pm$ 0.034 & 0.758 $\pm$ 0.008 \\
GP-CFM & 0.794 $\pm$ 0.004 & 0.963 $\pm$ 0.004 & 0.866 $\pm$ 0.008 & 0.879 $\pm$ 0.006 & 0.730 $\pm$ 0.013 & 0.855 $\pm$ 0.005 \\
MMFM & 0.793 $\pm$ 0.020 & 0.969 $\pm$ 0.003 & 0.872 $\pm$ 0.007 & 0.882 $\pm$ 0.007 & 0.744 $\pm$ 0.007 & 0.861 $\pm$ 0.008 \\
\midrule
\texttt{EHRFlow} (GP-CFM) & \textbf{0.817} $\pm$ 0.012 & \textbf{0.969} $\pm$ 0.007 & \textbf{0.889} $\pm$ 0.007 & \textbf{0.934} $\pm$ 0.004 & 0.829 $\pm$ 0.010 & \textbf{0.864} $\pm$ 0.004 \\
\texttt{EHRFlow} (MMFM) & 0.789 $\pm$ 0.024 & 0.968 $\pm$ 0.003 & 0.881 $\pm$ 0.014 & 0.932 $\pm$ 0.004 & \textbf{0.837} $\pm$ 0.011 & 0.862 $\pm$ 0.004 \\
\bottomrule
\end{tabular}}
\end{table}

\begin{figure}
    \centering
    \includegraphics[width=\linewidth]{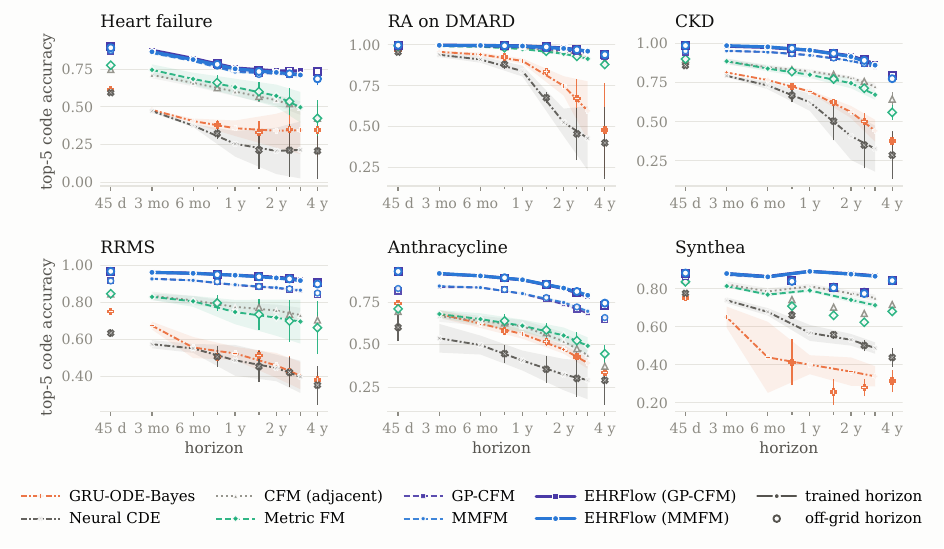}
    \caption{Clinical-code top-5 accuracy at trained and off-grid forecast horizons across the synthetic benchmarks. Off-grid evaluation tests interpolation between training horizons and extrapolation outside the training range. Results are averaged across seeds and shown with $\pm$ 1 s.d.}
    \label{fig:offgrid}
\end{figure}

\begin{table}[htbp]
\centering
\caption{Trajectory self-consistency per synthetic dataset and nested pair ($h$ read off a rollout to $H$). Direct: top-5 code accuracy of a rollout to day $h$; Nested: the same day read off the rollout to $H$; Overlap: shared fraction of the two top-5 sets. Mean $\pm$ s.d. over seeds.}
\label{tab:trajectory_consistency_full}
\resizebox{\textwidth}{!}{%
\begin{tabular}{lcccccc}
\toprule
 & \multicolumn{3}{c}{\texttt{EHRFlow} (GP-CFM)} & \multicolumn{3}{c}{\texttt{EHRFlow} (MMFM)} \\
\cmidrule(lr){2-4}\cmidrule(lr){5-7}
 & Direct & Nested & Overlap & Direct & Nested & Overlap \\
\midrule
\multicolumn{7}{l}{\textbf{Heart failure}} \\
\quad 90 d in 180 d & 0.872 $\pm$ 0.007 & 0.876 $\pm$ 0.008 & 0.952 $\pm$ 0.012 & 0.867 $\pm$ 0.012 & 0.870 $\pm$ 0.014 & 0.940 $\pm$ 0.015 \\
\quad 90 d in 365 d & 0.872 $\pm$ 0.007 & 0.874 $\pm$ 0.008 & 0.929 $\pm$ 0.017 & 0.867 $\pm$ 0.012 & 0.869 $\pm$ 0.009 & 0.913 $\pm$ 0.015 \\
\quad 180 d in 365 d & 0.811 $\pm$ 0.004 & 0.817 $\pm$ 0.005 & 0.936 $\pm$ 0.006 & 0.807 $\pm$ 0.020 & 0.813 $\pm$ 0.018 & 0.928 $\pm$ 0.018 \\
\quad 180 d in 730 d & 0.811 $\pm$ 0.004 & 0.820 $\pm$ 0.006 & 0.909 $\pm$ 0.015 & 0.807 $\pm$ 0.020 & 0.819 $\pm$ 0.006 & 0.896 $\pm$ 0.028 \\
\quad 365 d in 730 d & 0.743 $\pm$ 0.008 & 0.755 $\pm$ 0.009 & 0.915 $\pm$ 0.005 & 0.749 $\pm$ 0.019 & 0.759 $\pm$ 0.015 & 0.922 $\pm$ 0.013 \\
\quad 365 d in 1095 d & 0.743 $\pm$ 0.008 & 0.751 $\pm$ 0.011 & 0.876 $\pm$ 0.027 & 0.749 $\pm$ 0.019 & 0.760 $\pm$ 0.012 & 0.886 $\pm$ 0.027 \\
\quad \textit{Mean} & 0.809 $\pm$ 0.003 & 0.815 $\pm$ 0.006 & 0.920 $\pm$ 0.011 & 0.808 $\pm$ 0.017 & 0.815 $\pm$ 0.012 & 0.914 $\pm$ 0.018 \\
\midrule
\multicolumn{7}{l}{\textbf{RA on DMARD}} \\
\quad 90 d in 180 d & 0.998 $\pm$ 0.000 & 0.998 $\pm$ 0.000 & 0.962 $\pm$ 0.006 & 0.998 $\pm$ 0.001 & 0.998 $\pm$ 0.000 & 0.961 $\pm$ 0.010 \\
\quad 90 d in 365 d & 0.998 $\pm$ 0.000 & 0.997 $\pm$ 0.001 & 0.938 $\pm$ 0.010 & 0.998 $\pm$ 0.001 & 0.997 $\pm$ 0.001 & 0.935 $\pm$ 0.017 \\
\quad 180 d in 365 d & 0.997 $\pm$ 0.001 & 0.997 $\pm$ 0.000 & 0.955 $\pm$ 0.009 & 0.997 $\pm$ 0.001 & 0.996 $\pm$ 0.000 & 0.954 $\pm$ 0.015 \\
\quad 180 d in 730 d & 0.997 $\pm$ 0.001 & 0.996 $\pm$ 0.001 & 0.932 $\pm$ 0.011 & 0.997 $\pm$ 0.001 & 0.995 $\pm$ 0.001 & 0.930 $\pm$ 0.020 \\
\quad 365 d in 730 d & 0.994 $\pm$ 0.001 & 0.993 $\pm$ 0.001 & 0.951 $\pm$ 0.012 & 0.994 $\pm$ 0.001 & 0.992 $\pm$ 0.001 & 0.955 $\pm$ 0.012 \\
\quad 365 d in 1095 d & 0.994 $\pm$ 0.001 & 0.993 $\pm$ 0.000 & 0.940 $\pm$ 0.012 & 0.994 $\pm$ 0.001 & 0.991 $\pm$ 0.001 & 0.942 $\pm$ 0.015 \\
\quad \textit{Mean} & 0.996 $\pm$ 0.000 & 0.996 $\pm$ 0.000 & 0.946 $\pm$ 0.008 & 0.996 $\pm$ 0.001 & 0.995 $\pm$ 0.001 & 0.946 $\pm$ 0.014 \\
\midrule
\multicolumn{7}{l}{\textbf{CKD}} \\
\quad 90 d in 180 d & 0.984 $\pm$ 0.002 & 0.987 $\pm$ 0.001 & 0.960 $\pm$ 0.020 & 0.986 $\pm$ 0.002 & 0.988 $\pm$ 0.001 & 0.964 $\pm$ 0.005 \\
\quad 90 d in 365 d & 0.984 $\pm$ 0.002 & 0.983 $\pm$ 0.003 & 0.946 $\pm$ 0.028 & 0.986 $\pm$ 0.002 & 0.984 $\pm$ 0.001 & 0.946 $\pm$ 0.009 \\
\quad 180 d in 365 d & 0.976 $\pm$ 0.003 & 0.979 $\pm$ 0.001 & 0.964 $\pm$ 0.019 & 0.978 $\pm$ 0.004 & 0.980 $\pm$ 0.001 & 0.963 $\pm$ 0.006 \\
\quad 180 d in 730 d & 0.976 $\pm$ 0.003 & 0.976 $\pm$ 0.002 & 0.949 $\pm$ 0.029 & 0.978 $\pm$ 0.004 & 0.977 $\pm$ 0.002 & 0.944 $\pm$ 0.008 \\
\quad 365 d in 730 d & 0.957 $\pm$ 0.004 & 0.961 $\pm$ 0.002 & 0.961 $\pm$ 0.010 & 0.957 $\pm$ 0.003 & 0.962 $\pm$ 0.003 & 0.953 $\pm$ 0.011 \\
\quad 365 d in 1095 d & 0.957 $\pm$ 0.004 & 0.959 $\pm$ 0.002 & 0.948 $\pm$ 0.017 & 0.957 $\pm$ 0.003 & 0.960 $\pm$ 0.003 & 0.935 $\pm$ 0.013 \\
\quad \textit{Mean} & 0.972 $\pm$ 0.002 & 0.974 $\pm$ 0.001 & 0.955 $\pm$ 0.020 & 0.974 $\pm$ 0.003 & 0.975 $\pm$ 0.001 & 0.951 $\pm$ 0.006 \\
\midrule
\multicolumn{7}{l}{\textbf{RRMS}} \\
\quad 90 d in 180 d & 0.963 $\pm$ 0.003 & 0.968 $\pm$ 0.002 & 0.962 $\pm$ 0.010 & 0.961 $\pm$ 0.004 & 0.965 $\pm$ 0.004 & 0.947 $\pm$ 0.012 \\
\quad 90 d in 365 d & 0.963 $\pm$ 0.003 & 0.962 $\pm$ 0.005 & 0.942 $\pm$ 0.010 & 0.961 $\pm$ 0.004 & 0.962 $\pm$ 0.006 & 0.925 $\pm$ 0.015 \\
\quad 180 d in 365 d & 0.959 $\pm$ 0.002 & 0.962 $\pm$ 0.003 & 0.969 $\pm$ 0.003 & 0.956 $\pm$ 0.006 & 0.958 $\pm$ 0.004 & 0.958 $\pm$ 0.012 \\
\quad 180 d in 730 d & 0.959 $\pm$ 0.002 & 0.955 $\pm$ 0.003 & 0.945 $\pm$ 0.005 & 0.956 $\pm$ 0.006 & 0.956 $\pm$ 0.005 & 0.933 $\pm$ 0.022 \\
\quad 365 d in 730 d & 0.948 $\pm$ 0.001 & 0.948 $\pm$ 0.002 & 0.971 $\pm$ 0.004 & 0.942 $\pm$ 0.009 & 0.945 $\pm$ 0.005 & 0.956 $\pm$ 0.018 \\
\quad 365 d in 1095 d & 0.948 $\pm$ 0.001 & 0.945 $\pm$ 0.003 & 0.955 $\pm$ 0.007 & 0.942 $\pm$ 0.009 & 0.943 $\pm$ 0.007 & 0.938 $\pm$ 0.029 \\
\quad \textit{Mean} & 0.957 $\pm$ 0.001 & 0.957 $\pm$ 0.002 & 0.957 $\pm$ 0.005 & 0.953 $\pm$ 0.006 & 0.955 $\pm$ 0.004 & 0.943 $\pm$ 0.017 \\
\midrule
\multicolumn{7}{l}{\textbf{Anthracycline}} \\
\quad 90 d in 180 d & 0.915 $\pm$ 0.006 & 0.917 $\pm$ 0.006 & 0.940 $\pm$ 0.010 & 0.920 $\pm$ 0.005 & 0.925 $\pm$ 0.003 & 0.941 $\pm$ 0.008 \\
\quad 90 d in 365 d & 0.915 $\pm$ 0.006 & 0.902 $\pm$ 0.004 & 0.908 $\pm$ 0.014 & 0.920 $\pm$ 0.005 & 0.909 $\pm$ 0.005 & 0.906 $\pm$ 0.011 \\
\quad 180 d in 365 d & 0.904 $\pm$ 0.006 & 0.905 $\pm$ 0.007 & 0.937 $\pm$ 0.012 & 0.899 $\pm$ 0.007 & 0.906 $\pm$ 0.005 & 0.932 $\pm$ 0.012 \\
\quad 180 d in 730 d & 0.904 $\pm$ 0.006 & 0.890 $\pm$ 0.006 & 0.904 $\pm$ 0.017 & 0.899 $\pm$ 0.007 & 0.896 $\pm$ 0.007 & 0.897 $\pm$ 0.012 \\
\quad 365 d in 730 d & 0.879 $\pm$ 0.009 & 0.881 $\pm$ 0.012 & 0.933 $\pm$ 0.010 & 0.864 $\pm$ 0.009 & 0.879 $\pm$ 0.007 & 0.919 $\pm$ 0.010 \\
\quad 365 d in 1095 d & 0.879 $\pm$ 0.009 & 0.870 $\pm$ 0.011 & 0.909 $\pm$ 0.017 & 0.864 $\pm$ 0.009 & 0.873 $\pm$ 0.011 & 0.890 $\pm$ 0.011 \\
\quad \textit{Mean} & 0.900 $\pm$ 0.006 & 0.894 $\pm$ 0.007 & 0.922 $\pm$ 0.012 & 0.894 $\pm$ 0.007 & 0.898 $\pm$ 0.005 & 0.914 $\pm$ 0.010 \\
\midrule
\multicolumn{7}{l}{\textbf{Synthea}} \\
\quad 90 d in 180 d & 0.889 $\pm$ 0.004 & 0.888 $\pm$ 0.003 & 0.945 $\pm$ 0.020 & 0.886 $\pm$ 0.008 & 0.887 $\pm$ 0.006 & 0.935 $\pm$ 0.014 \\
\quad 90 d in 365 d & 0.889 $\pm$ 0.004 & 0.887 $\pm$ 0.003 & 0.925 $\pm$ 0.021 & 0.886 $\pm$ 0.008 & 0.888 $\pm$ 0.007 & 0.910 $\pm$ 0.021 \\
\quad 180 d in 365 d & 0.876 $\pm$ 0.003 & 0.878 $\pm$ 0.004 & 0.945 $\pm$ 0.017 & 0.873 $\pm$ 0.003 & 0.877 $\pm$ 0.004 & 0.935 $\pm$ 0.010 \\
\quad 180 d in 730 d & 0.876 $\pm$ 0.003 & 0.877 $\pm$ 0.003 & 0.923 $\pm$ 0.020 & 0.873 $\pm$ 0.003 & 0.877 $\pm$ 0.004 & 0.911 $\pm$ 0.018 \\
\quad 365 d in 730 d & 0.900 $\pm$ 0.004 & 0.900 $\pm$ 0.005 & 0.936 $\pm$ 0.015 & 0.900 $\pm$ 0.004 & 0.900 $\pm$ 0.004 & 0.921 $\pm$ 0.007 \\
\quad 365 d in 1095 d & 0.900 $\pm$ 0.004 & 0.900 $\pm$ 0.005 & 0.921 $\pm$ 0.013 & 0.900 $\pm$ 0.004 & 0.900 $\pm$ 0.005 & 0.905 $\pm$ 0.009 \\
\quad \textit{Mean} & 0.888 $\pm$ 0.004 & 0.888 $\pm$ 0.004 & 0.933 $\pm$ 0.017 & 0.886 $\pm$ 0.005 & 0.888 $\pm$ 0.005 & 0.920 $\pm$ 0.012 \\
\midrule
\multicolumn{7}{l}{\textbf{Average over datasets}} \\
\quad 90 d in 180 d & 0.937 $\pm$ 0.002 & 0.939 $\pm$ 0.001 & 0.954 $\pm$ 0.007 & 0.936 $\pm$ 0.003 & 0.939 $\pm$ 0.003 & 0.948 $\pm$ 0.004 \\
\quad 90 d in 365 d & 0.937 $\pm$ 0.002 & 0.934 $\pm$ 0.003 & 0.931 $\pm$ 0.008 & 0.936 $\pm$ 0.003 & 0.935 $\pm$ 0.002 & 0.922 $\pm$ 0.004 \\
\quad 180 d in 365 d & 0.921 $\pm$ 0.001 & 0.923 $\pm$ 0.001 & 0.951 $\pm$ 0.006 & 0.918 $\pm$ 0.004 & 0.922 $\pm$ 0.003 & 0.945 $\pm$ 0.004 \\
\quad 180 d in 730 d & 0.921 $\pm$ 0.001 & 0.919 $\pm$ 0.002 & 0.927 $\pm$ 0.007 & 0.918 $\pm$ 0.004 & 0.920 $\pm$ 0.002 & 0.919 $\pm$ 0.006 \\
\quad 365 d in 730 d & 0.904 $\pm$ 0.002 & 0.906 $\pm$ 0.002 & 0.945 $\pm$ 0.004 & 0.901 $\pm$ 0.004 & 0.906 $\pm$ 0.003 & 0.938 $\pm$ 0.004 \\
\quad 365 d in 1095 d & 0.904 $\pm$ 0.002 & 0.903 $\pm$ 0.002 & 0.925 $\pm$ 0.007 & 0.901 $\pm$ 0.004 & 0.904 $\pm$ 0.003 & 0.916 $\pm$ 0.007 \\
\quad \textit{Mean} & 0.920 $\pm$ 0.000 & 0.921 $\pm$ 0.001 & 0.939 $\pm$ 0.006 & 0.919 $\pm$ 0.003 & 0.921 $\pm$ 0.002 & 0.931 $\pm$ 0.005 \\
\bottomrule
\end{tabular}}
\end{table}

\subsection{History-conditioning ablations}

We extend the aggregate ablations in the main text with dataset-level results and an evaluation of how performance changes with the amount of historical context.

\paragraph{Dataset-level ablations.}
Table~\ref{tab:ablations_by_dataset} separates the effects of model capacity, history availability, optimisation, and the way historical information enters the flow. Matching the vector-field parameter count largely preserves performance, whereas replacing the history representation with zeros returns results close to the history-independent models. This pattern is particularly clear on Anthracycline and RRMS, supporting the interpretation that the gains arise from informative historical context rather than additional vector-field capacity.

Using a learned history representation as the flow source performs less well than conditioning a flow initialised from the encoded anchor state. Jointly training the history encoder with the flow also generally underperforms pretraining and freezing it under the evaluated optimisation settings. Fully end-to-end training likewise underperforms the staged procedure across all datasets for both interpolant families, with increased variability on several benchmarks. Removing the auxiliary decoder loss substantially reduces accuracy across datasets. This ablation removes observation-space supervision and leaves the decoder frozen, so it establishes the importance of this combined training choice rather than isolating either component individually. Finally, adding history to adjacent CFM improves average performance but remains below the multi-marginal variants.
\begin{table}[!htb]
\centering
\caption{History-conditioning ablations per dataset: top-5 clinical-code accuracy averaged over horizons (mean $\pm$ std over seeds).}
\label{tab:ablations_by_dataset}
\resizebox{\textwidth}{!}{%
\begin{tabular}{lcccccc}
\toprule
 & Anthracycline & RA on DMARD & CKD & RRMS & Heart failure & Synthea \\
\midrule
\multicolumn{7}{l}{\textit{MMFM path}} \\
\texttt{EHRFlow} & \textbf{0.865} $\pm$ 0.007 & 0.985 $\pm$ 0.002 & 0.940 $\pm$ 0.004 & 0.942 $\pm$ 0.004 & 0.775 $\pm$ 0.009 & 0.875 $\pm$ 0.003 \\
\texttt{EHRFlow} param. matched & 0.858 $\pm$ 0.004 & 0.986 $\pm$ 0.001 & 0.938 $\pm$ 0.005 & 0.940 $\pm$ 0.006 & 0.784 $\pm$ 0.010 & 0.877 $\pm$ 0.004 \\
\midrule
\texttt{EHRFlow}, zero context & 0.783 $\pm$ 0.005 & 0.979 $\pm$ 0.002 & 0.910 $\pm$ 0.004 & 0.898 $\pm$ 0.003 & 0.766 $\pm$ 0.011 & 0.876 $\pm$ 0.002 \\
\texttt{EHRFlow}, end-to-end GRU & 0.807 $\pm$ 0.014 & 0.982 $\pm$ 0.002 & 0.926 $\pm$ 0.007 & 0.919 $\pm$ 0.005 & 0.757 $\pm$ 0.010 & 0.874 $\pm$ 0.003 \\
\texttt{EHRFlow}, fully end-to-end & 0.767 $\pm$ 0.021 & 0.981 $\pm$ 0.002 & 0.919 $\pm$ 0.006 & 0.911 $\pm$ 0.011 & 0.748 $\pm$ 0.009 & 0.861 $\pm$ 0.008 \\
\texttt{EHRFlow}, no decoder loss & 0.460 $\pm$ 0.023 & 0.834 $\pm$ 0.015 & 0.691 $\pm$ 0.036 & 0.475 $\pm$ 0.027 & 0.474 $\pm$ 0.025 & 0.703 $\pm$ 0.012 \\
History as source & 0.780 $\pm$ 0.007 & 0.978 $\pm$ 0.002 & 0.910 $\pm$ 0.003 & 0.892 $\pm$ 0.009 & 0.763 $\pm$ 0.004 & 0.878 $\pm$ 0.003 \\
History as source + context & 0.773 $\pm$ 0.004 & 0.979 $\pm$ 0.002 & 0.905 $\pm$ 0.006 & 0.886 $\pm$ 0.008 & 0.770 $\pm$ 0.010 & 0.877 $\pm$ 0.004 \\
No history & 0.785 $\pm$ 0.003 & 0.980 $\pm$ 0.001 & 0.913 $\pm$ 0.004 & 0.896 $\pm$ 0.003 & 0.767 $\pm$ 0.004 & 0.874 $\pm$ 0.003 \\
\midrule
\multicolumn{7}{l}{\textit{GP-CFM path}} \\
\texttt{EHRFlow} & 0.863 $\pm$ 0.007 & \textbf{0.986} $\pm$ 0.002 & 0.941 $\pm$ 0.001 & 0.944 $\pm$ 0.005 & \textbf{0.787} $\pm$ 0.004 & 0.875 $\pm$ 0.003 \\
\texttt{EHRFlow} param. matched & 0.863 $\pm$ 0.003 & 0.984 $\pm$ 0.002 & \textbf{0.941} $\pm$ 0.003 & \textbf{0.944} $\pm$ 0.000 & 0.772 $\pm$ 0.010 & 0.875 $\pm$ 0.002 \\
\midrule
\texttt{EHRFlow}, zero context & 0.780 $\pm$ 0.012 & 0.980 $\pm$ 0.002 & 0.913 $\pm$ 0.002 & 0.894 $\pm$ 0.002 & 0.768 $\pm$ 0.007 & 0.873 $\pm$ 0.002 \\
\texttt{EHRFlow}, end-to-end GRU & 0.793 $\pm$ 0.019 & 0.979 $\pm$ 0.002 & 0.919 $\pm$ 0.011 & 0.913 $\pm$ 0.009 & 0.761 $\pm$ 0.009 & 0.871 $\pm$ 0.005 \\
\texttt{EHRFlow}, fully end-to-end & 0.805 $\pm$ 0.021 & 0.983 $\pm$ 0.003 & 0.926 $\pm$ 0.006 & 0.909 $\pm$ 0.021 & 0.759 $\pm$ 0.015 & 0.874 $\pm$ 0.003 \\
\texttt{EHRFlow}, no decoder loss & 0.492 $\pm$ 0.033 & 0.829 $\pm$ 0.006 & 0.666 $\pm$ 0.028 & 0.453 $\pm$ 0.042 & 0.465 $\pm$ 0.043 & 0.690 $\pm$ 0.028 \\
History as source & 0.775 $\pm$ 0.006 & 0.976 $\pm$ 0.001 & 0.908 $\pm$ 0.002 & 0.890 $\pm$ 0.009 & 0.764 $\pm$ 0.005 & \textbf{0.879} $\pm$ 0.003 \\
History as source + context & 0.775 $\pm$ 0.014 & 0.975 $\pm$ 0.003 & 0.899 $\pm$ 0.017 & 0.886 $\pm$ 0.006 & 0.772 $\pm$ 0.011 & 0.877 $\pm$ 0.003 \\
No history & 0.779 $\pm$ 0.004 & 0.977 $\pm$ 0.001 & 0.912 $\pm$ 0.003 & 0.895 $\pm$ 0.005 & 0.770 $\pm$ 0.005 & 0.875 $\pm$ 0.004 \\
\midrule
\multicolumn{7}{l}{\textit{CFM (linear)}} \\
CFM (adjacent) + history & 0.655 $\pm$ 0.013 & 0.957 $\pm$ 0.016 & 0.815 $\pm$ 0.026 & 0.834 $\pm$ 0.021 & 0.599 $\pm$ 0.040 & 0.799 $\pm$ 0.026 \\
CFM (adjacent) & 0.574 $\pm$ 0.035 & 0.971 $\pm$ 0.003 & 0.812 $\pm$ 0.014 & 0.781 $\pm$ 0.024 & 0.602 $\pm$ 0.026 & 0.790 $\pm$ 0.014 \\
CFM (MMFM windows) + history & 0.743 $\pm$ 0.020 & 0.963 $\pm$ 0.009 & 0.871 $\pm$ 0.021 & 0.860 $\pm$ 0.021 & 0.585 $\pm$ 0.045 & 0.830 $\pm$ 0.013 \\
\bottomrule
\end{tabular}}
\end{table}

\paragraph{Sensitivity to history-window length.}
Figure~\ref{fig:history_window} varies the number of day-states $K$ available to the history encoder and reports top-5 accuracy averaged across forecast horizons. On Anthracycline and RRMS, performance improves as historical observations are added and then broadly plateaus. CKD shows an earlier improvement followed by relatively little sensitivity to longer windows. In contrast, RA on DMARD and Synthea vary little over the evaluated window lengths, while Heart failure shows modest, non-monotonic changes. The similar patterns across the two interpolant families suggest that sensitivity to context length is primarily associated with the available historical information rather than the particular path construction. The default window of $K=32$ lies within the region of broadly stable performance across these benchmarks. 

\begin{figure}[!htb]
    \centering
    \includegraphics[width=\linewidth]{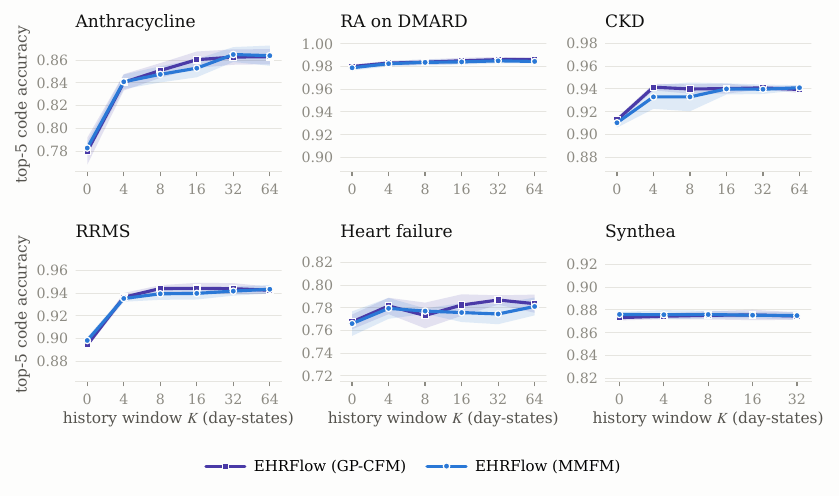}
    \caption{Top-5 clinical-code accuracy, averaged across forecast horizons, is shown for each dataset as the number of day-states $K$ available to the history encoder varies. Lines show the mean across random seeds; shaded bands indicate $\pm$ one standard deviation.}
    \label{fig:history_window}
\end{figure}

\newpage

\section{Experimental details}
\label{app:experimental_details}
\subsection{Synthetic datasets}
\label{app:synthetic}
To evaluate the benefits of history conditioning in a controlled environment with a ground truth latent trajectory, we constructed five datasets with different history mechanisms. Each dataset provides daily simulation of 20{,}000 patients for up to five years with no mortality events. Each patient $i$ carries a two-dimensional latent state
\begin{equation}
z_{i,t} = \big(z^{(s)}_{i,t},\, z^{(d)}_{i,t}\big),
\end{equation}
which defines a disease-severity axis and a disease-domain axis, together with an unobserved
scalar baseline $b_i\sim\mathcal{N}(0,1)$ and initial state
$z^{(s)}_{i,0}\sim\mathcal{N}(0.25\,b_i,\,0.35^2)$,
$z^{(d)}_{i,0}\sim\mathcal{N}(0,\,0.25^2)$. To simulate patient evolution, we perform an Euler--Maruyama step
\begin{equation}
z_{i,t+1}
=
\operatorname{clip}\!\left(
z_{i,t}
+ \kappa\, f(z_{i,t},a_{i,t},b_i)\,\Delta t
+ \eta\, \sigma(z_{i,t})\,\xi_{i,t}\sqrt{\Delta t},\;
-3,\, 6
\right),
\qquad
\xi_{i,t}\sim\mathcal{N}(0,I_2),
\end{equation}
with $\Delta t = 1/365.25$ years, treatment indicator $a_{i,t}\in\{0,1\}$, and
dataset-level multipliers $\kappa$ and $\eta$ on the whole drift and diffusion (default 1). The shared drift and diffusion are
\begin{align}
f^{(s)} &= r_i\,(0.35 + 0.45\,b_i)\,s_i + 0.25\,z^{(d)} - \tau\, a, \\
f^{(d)} &= -0.25\,z^{(d)} + 0.12\,z^{(s)} + 0.05\,b_i, \\
\sigma(z) &= \operatorname{diag}\!\big(0.08 + 0.03\max(z^{(s)},0),\; 0.06 + 0.02\,|z^{(d)}|\big),
\end{align}
where $r_i$ is the patient-specific progression rate (1 except in CKD),
$s_i$ is a fixed per-patient progression sign drawn once at baseline
($s_i=+1$ with probability $0.35$, progressing; $s_i=-0.6$ otherwise,
remitting), and $\tau=0.25$ is the instantaneous treatment effect.

\paragraph{Observation process.}
We model visits with  probability $\min(0.95,\, 10\,\Delta t)$. A visit emits an outpatient encounter, a diagnosis code and two laboratory values.
The diagnosis is read off a $5\times 3$ grid over the latent state: the
severity bin is the number of edges in $\{-0.5,0.5,1.5,2.5\}$ that lie at or
below $z^{(s)}$ and the domain $\in\{\text{renal},\text{cardiac},\text{metabolic}\}$
is read off $z^{(d)}$ with edges $\{-0.3, 0.3\}$, giving 15 diagnosis codes.
The laboratory values are $50 + 35\,z^{(s)} + \mathcal{N}(0,8^2)$ (severity
marker) and $1 + 0.25\,z^{(d)} + 0.15\,a + \mathcal{N}(0,0.08^2)$ (renal
marker). Treatment status is redrawn every 30 days as
$a\sim\operatorname{Bernoulli}\big(\operatorname{sigmoid}(-2 + 0.25\,z^{(s)} + 1.5\,a + \delta\,\mathbb{1}[a=0])\big)$,
where $\delta$ is a dataset-level start-logit shift (0 unless stated), so
patients both start and stop treatment; each start and stop is written as a
medication event.

Below, we describe the mechanisms used to introduce dependence on patient history. These synthetic datasets are inspired by broad clinical patterns, such as cumulative treatment exposure and delayed treatment response, but serve as simplified illustrations of history-dependent dynamics rather than validated mechanistic models of disease progression or treatment effects in real patients.

\paragraph{Heart failure: Markov dynamics.}
Severity follows a five-level ladder with set-points $\mu_k = -1 + k$,
$k\in\{0,\ldots,4\}$; the initial level is read off $z^{(s)}_{0}$ with the
diagnosis-grid edges. On each day, with probability $1/100$, the level moves to
an adjacent level, upward with probability $0.7$ and downward with probability
$0.3$, reflecting at the boundaries. 

\paragraph{RA on DMARD: delayed treatment effect.}
Treatment initiation at $t_s$ is observed immediately (medication event) and is
absorbing, but the latent responds only inside a lagged window,
\begin{equation}
a^{\mathrm{eff}}_t=\mathbb{1}\!\left[t_s+270 \le t < t_s+670\right],
\end{equation}
during which the treatment term in $f^{(s)}$ is $-3.5\,a^{\mathrm{eff}}_t$;
outside it the treatment has no effect on the latent (the renal marker still
carries the $0.15\,a$ term). This dataset uses $\kappa=0.3$, $\eta=0.7$ and a
start-logit shift $\delta=-1.5$.

\paragraph{CKD: patient-specific progression rate.}
Each patient has a fixed unobserved rate
\begin{equation}
r_i=\operatorname{clip}\!\left(\operatorname{LogNormal}\!\left(-\tfrac{1}{2}\,0.9^2,\,0.9\right),\;\tfrac14,\;4\right),
\end{equation}
which multiplies only the progression term $(0.35+0.45\,b_i)\,s_i$ of
$f^{(s)}$, not the coupling or treatment terms. This dataset uses
$\kappa=0.9$, $\eta=0.7$. 

\paragraph{RRMS: relapse history.}
Patients sit at a baseline set-point $\mu_i^{\mathrm{base}}\in\{-1,0,1\}$
(chosen from the initial state) and experience 90-day flares during which the
set-point rises by 2 severity bins. The first onset is uniform on the first
270 days; subsequent onsets form a renewal process with
\begin{equation}
T_{k+1}-T_k \;=\; \max\!\big(91,\ \operatorname{round}(270\cdot G_k)\big)\ \text{days},\qquad
G_k\sim\Gamma(\text{shape}=64,\ \text{scale}=1/64),
\end{equation}
i.e.\ mean 270 days and coefficient of variation $1/\sqrt{64}=0.125$. 

\paragraph{Anthracycline: cumulative exposure.}
Thirty-day treatment episodes start with a per-patient daily hazard
$h_i = \rho_i/180$, $\rho_i=\operatorname{clip}(\operatorname{LogNormal}(-\tfrac12 0.7^2, 0.7), \tfrac14, 4)$,
while no episode is active, and each start and stop is written as a medication alongside the shared
treatment review above. Exposure accumulates as
\begin{equation}
E_t=\sum_{u\le t}\mathbb{1}[\text{exposed at }u],
\end{equation}
and once $E_t \ge 180$ days the set-point rises permanently by 2 severity bins,
\begin{equation}
\mu_t=\mu_i^{\mathrm{base}} + 2\cdot\mathbb{1}[E_t\ge 180],
\end{equation}
with $\mu_i^{\mathrm{base}}\in\{-1,0,1\}$ as in RRMS.

\begin{table}[t]
\vspace{-12pt}
\centering
\caption{Synthetic datasets and the source of history dependence. Event and
code counts are those of the processed data the models are trained on.}
\vspace{2pt}
\label{tab:synthetic-datasets}
\small
\begin{tabular}{lrrrl}
\toprule
Dataset & Patients & Events & Codes & Predictive information \\
\midrule
Heart failure & 20{,}000 & 5{,}311{,}935 & 20 & Current state only  \\
RA on DMARD & 20{,}000 & 5{,}013{,}692 & 18 & Time since treatment initiation \\
CKD & 20{,}000 & 5{,}234{,}333 & 20 & Historical progression rate \\
RRMS & 20{,}000 & 5{,}264{,}932 & 20 & Time since previous relapse \\
Anthracycline & 20{,}000 & 5{,}587{,}408 & 22 & Cumulative treatment exposure \\
Synthea & 23{,}072 & 20{,}586{,}237 & 1{,}893 & Rule-based longitudinal disease history \\
\bottomrule
\end{tabular}
\vspace{-12pt}
\end{table}
\subsection{Synthea.}
For a more realistic patient simulator we also generated a cohort of 23{,}072 patients with Synthea v3.4.0 \citep{Walonoski2017}.  Unlike the controlled datasets, Synthea is produced by rule-based clinical state machines rather than our latent simulator. We retain conditions, hospitalisations (inpatient, emergency and
urgent-care encounters), observations (laboratory results, vital signs and
survey items), medication starts and stops, and deaths; all other encounter
classes (ambulatory, wellness, outpatient, home, virtual) are dropped. A summary of the different datasets is provided in Table~\ref{tab:synthetic-datasets}.

\subsection{Real-world datasets}

\paragraph{CPRD.}
We use a heart-failure cohort from the UK Clinical Practice Research Datalink
\citep{Herrett2015} covering both the Aurum and Gold primary-care databases:
868{,}691 patients (Aurum 815{,}092, Gold 53{,}599) with heart-failure index
dates between 2006 and 2023, and 650{,}147{,}233 events. Event sources are multimorbidity onsets, laboratory
tests, medications (drug and drug-class codes), BMI, blood pressure, smoking
and alcohol status, frailty scores, NYHA class, heart-failure phenotype,
A\&E visits, hospital admissions and admission diagnoses, and mortality, plus
three static demographic rows per patient (sex, ethnicity, deprivation). We train on Aurum, hold out a random 15\% of Aurum patients for
validation and epoch selection, and report model performance on all 53{,}599 Gold patients
(35{,}758{,}467 events) as an external test set. This study was conducted in accordance with the Declaration of Helsinki
and was approved through CPRD's Research Data Governance (RDG) process
(protocol 23\_002943).

\paragraph{AMC-EHR.}
The dataset is derived from the de-identified STAnford medicine Research
data Repository Observational Medical Outcomes Partnership
(STARR-OMOP) database \citep{datta2020starr}, comprising Stanford Medicine
electronic health records represented in the OMOP Common Data Model (CDM).
The dataset contains longitudinal data from a large tertiary academic medical
center in the USA (outpatient visits, hospital admissions, emergency visits,
observation-stay visits, conditions, drugs, measurements, and mortality).
We included only patients with at least five events to construct a full
trajectory, yielding 529{,}125{,}023 events across 647{,}447 patients.
Splits are patient-level (15\% validation, 15\% test). The analysis used the de-identified STARR-OMOP database, which Stanford
designates as a pre-IRB, non-human-subject dataset accessible to researchers
without IRB approval.

\subsection{Event representation}
Every dataset is converted to one normalised event table
(patient, date, source, code, optional value). The flow and continuous-time
models consume a fixed per-event feature vector: a source one-hot, a code
one-hot over the 511 most frequent codes plus an unknown token, the
standardised value and a has-value flag, four standardised time columns (age,
time since registration, time since index diagnosis, days since the previous
event), two indicator flags (action event, terminal event) and a constant
bias column. Days are binned  so a patient trajectory is a sequence of day-states.

\subsection{\texttt{EHRFlow} implementation}
\label{app:ehrflow-impl}

\paragraph{Latent autoencoder.}
A variational autoencoder maps each day-state to a latent
$z\in\mathbb{R}^{d}$. Encoder and decoder each have one hidden layer with SiLU
activation and dropout 0.05; the encoder outputs mean and log-variance (clamped
to $[-10,10]$), the decoder emits source logits, code logits and the continuous
block. Reconstruction loss is cross-entropy on source and code plus
$0.1\times$ MSE on the continuous block; the KL term has weight $10^{-3}$. The
autoencoder is pretrained for 3 epochs with AdamW; the encoder is then frozen
and the latent mean is used at evaluation. The decoder keeps training during
the flow stage.

\paragraph{Vector field.}
The velocity network is an MLP with three hidden layers (SiLU, dropout 0.05)
over $(z_t,\, s,\, \log(1+\Delta t),\, c)$, where $s\in[0,1]$ is path time,
$\Delta t$ the calendar gap of the pair in days and $c$ the history context
(absent on the no-history arms). Parameters are as follows: hidden 488 / latent 64 on the synthetic datasets,
567 / 256 on CPRD, 614 / 64 on AMC-EHR.

\paragraph{History encoder.}
A single-layer GRU (hidden 128) reads the $K=32$ day-states ending at the
anchor (the anchor and its 31 predecessors) and a linear layer projects its
final state to a 32-dimensional context. The GRU is pretrained for 2 epochs on
cross-entropy and then frozen, so the flow stage receives
a fixed history summary.

\paragraph{Training.}
Pairs are anchor/target day-states from the same patient. Flow models sample 4 knots per path at least 2\% of the window apart. The GP path draws the interpolant between knots from a Gaussian process with a linear endpoint mean and an RBF kernel on the residual. The spline path (MMFM) is a natural cubic spline through the same knots with the bridge standard deviation of  \citet{rohbeck2025modeling},
$\sigma_t=\sigma_{\text{base}}\,4(t_{k+1}-t)(t-t_k)/\Delta_k^2$,
$\sigma_{\text{base}}=0.1$. The training loss is the MSE between the predicted and target path velocities (with the target detached), plus, with weight 1, the observation-space reconstruction loss of the decoded one-step latent endpoint estimate
$D_\phi\!\left(z_t+(1-t)\,v_\theta(z_t,t,\lambda,h_a)\right)$
against the final observation of the window.

\paragraph{Optimisation and inference.}
We used AdamW with learning rate $10^{-3}$, weight decay $10^{-6}$ and gradient-norm clip 1.0.
On the synthetic datasets, we used batch size 512 and 16 epochs; on CPRD and AMC-EHR, we used batch size 4{,}096 and 5 epochs. The epoch with the best validation loss is selected for all models. Trajectories are integrated with 16 Euler-midpoint steps in path time from the
anchor latent to the horizon and decoded to source and code logits. 

\subsection{Baseline implementations}
\label{app:baselines}
We compare against models that learn longitudinal patient trajectories and support forecasting across multiple horizons. Task-specific predictors for a fixed clinical endpoint may achieve strong endpoint accuracy, but do not provide intermediate patient states, arbitrary-horizon forecasts, or trajectory rollout, and therefore do not assess the same modelling objective. Clinical-code prediction nevertheless provides a common endpoint metric for comparing forecasting accuracy across the trajectory models considered here.

\paragraph{CFM, MMFM, GP-CFM and Metric FM.}
All flow baselines are the identical network of
Appendix~\ref{app:ehrflow-impl} without the history encoder; they differ only
in the loss and pair sampler. CFM (adjacent) pairs consecutive day-states with
a straight-line path \citep{lipman2023flow}. Metric FM learns a geodesic-path network (2 layers of width 32,
trained for 5 epochs and frozen) under a LAND diagonal metric
$1/(\sum_j w_j\,dx_j^2+\rho)^{\alpha}$ with $\rho=0.01$, $\alpha=1$ and
bandwidth set to the median pairwise latent distance \citep{kapusniak2024metric}.

\paragraph{Continuous-time latent-state baselines (GRU-ODE-Bayes, Neural CDE).}
Both consume the same day-state feature vectors as the flow models and share
an observation head: a categorical distribution over the source block, a
categorical over the code block (soft targets, since day-states are
mean-aggregated) and a diagonal Gaussian over the continuous columns.
GRU-ODE-Bayes \citep{DeBrouwer2019} integrates the full GRU-ODE
$\mathrm{d}h/\mathrm{d}t=(1-z(h))(g(h)-h)$ between observations with Euler
steps of 7 days and applies the GRU-Bayes update at each observation, with the
paper's pre-jump likelihood plus a $10^{-4}$-weighted post-jump term. The
Neural CDE \citep{Kidger2020} follows the linearly interpolated day-state path
(time as the first channel) with a fixed-step RK4 solver and scores the next
observation by extrapolating the state over the gap with the features held
fixed; the vector field is a 2-layer MLP of width 256. Both are trained
on the observation negative log-likelihood of 64-day-state windows. Training uses AdamW with the flow arms'
learning rate and weight decay, a cosine schedule, and 16 epochs.

\paragraph{Decoder-only.}
A causal Transformer over event tokens: token embedding plus a continuous-time
encoding of (age$/100$, $\log$ gap$/10$) through a two-layer MLP, added to the
token embedding; pre-norm encoder layers with GELU and $4\times$ feed-forward
width; a causal mask in which same-day tokens cannot attend to each other and the diagonal is always open.
Following Delphi-2M there is no separate time head: each token logit is the
log-rate of an independent exponential risk, the total rate is
$\exp\operatorname{logsumexp}$ over non-padding tokens, and the waiting-time
NLL is weighted 1.0 against the next-token cross-entropy \citep{Shmatko2025}. One
\texttt{<no\_event>} filler token is inserted every five years of follow-up
from the patient's first event; fillers are
trained on but never scored. Vocabulary caps are 32 sources / 512 events.
Widths: hidden 116 with 4 layers and 4 heads on the synthetic datasets (144 on
Synthea), 190 with 4 layers and 2 heads on CPRD, 172 with 4 layers on AMC-EHR.

\paragraph{Encoder--decoder.}
Source and code embeddings (source width $\max(8,\text{hidden}/4)$, code
width the remainder) are concatenated to the hidden width; the eight numeric
columns of the event feature vector pass through a two-layer MLP and are
added, followed by a learned positional embedding and layer normalisation \citep{Zhang2026Apollo}.
Transformer encoder layers process the window and a single learned query is
decoded by cross-attention into one state vector that feeds the next-event
head and a separate exponential log-rate head for the gap, trained with the
same waiting-time NLL with one filler token per five years as
above. Widths: hidden 106 with 2 encoder and 2 decoder layers and 2 heads on
the synthetic datasets (132 on Synthea), 176 on CPRD, 160 on AMC-EHR.

\paragraph{Endpoints evaluation.}
For each horizon $H\in\{90,180,365,730,1095\}$ days an anchor day $t$ is paired with the first day at or after $t+H$
carrying a label-eligible event, provided it lies within
$\min(0.25H, 180)$ days of $t+H$. On the synthetic and Synthea datasets every
event day of an evaluation patient is an anchor and pairs are capped at
250{,}000 per horizon with at least one pair per patient. On CPRD and AMC-EHR a
single landmark per patient is chosen once for all horizons: the first event
with at least three (CPRD) or 32 (AMC-EHR) prior events and, on CPRD, at least
180 days after the heart-failure index date, so that every horizon scores the
same anchors. We test each (\texttt{EHRFlow} variant, baseline) comparison with a one-sided stratified sign-flip permutation test ($H_1$: \texttt{EHRFlow} higher), datasets as strata so that each dataset carries equal weight regardless of its scale, the statistic being the mean over strata of the
mean paired difference. Table cells report mean $\pm$ one standard deviation over
seeds unless otherwise specified.

\subsection{Ablations implementation}
\label{app:ablations-impl}
All ablation arms use the synthetic recipe (16 epochs, hidden 488 / latent 64 /
3 hidden layers, MMFM sampler with 4 knots, 5 seeds) and differ from \texttt{EHRFlow}
in one factor:
\begin{itemize}[leftmargin=*,itemsep=0.4ex,parsep=-0.1ex,topsep=0.25ex]
\item \textbf{No history}: the same network without the history encoder
(plain GP-CFM or MMFM).
\item \textbf{\texttt{EHRFlow}, end-to-end GRU}: the GRU is not pretrained; it is trained
jointly with the vector field through the flow loss.
\item \textbf{\texttt{EHRFlow}, fully end-to-end}: the latent encoder, decoder, history
encoder, and vector field are all trained jointly during flow training, rather
than freezing the latent and history encoders as in the standard \texttt{EHRFlow}
training procedure.
\item \textbf{\texttt{EHRFlow}, zero context}: the GRU is instantiated, so the field input
width and parameter count are identical to \texttt{EHRFlow}, but its output is
set to zero in both training and evaluation.
\item \textbf{\texttt{EHRFlow} matched}: hidden width 481, so that the field MLP has the
reference parameter count despite the wider input.
\item \textbf{\texttt{EHRFlow}, no decoder loss}: the decoder weight is set to $\beta=0$,
so the flow stage minimises the flow-matching loss alone. The decoder receives
no gradient and, like the encoder, stays fixed at its weights from latent
pretraining; everything else is \texttt{EHRFlow}.
\item \textbf{History as source}: a variational history encoder maps the last
$K=32$ day-states to a diagonal-Gaussian prior $q_\phi(z_0\mid\mathcal{H}_a)$,
KL-regularised with the same weight as the latent VAE, and a sample from it
replaces the encoded anchor as the starting point of the flow.  The vector
field is the unconditioned GP-CFM / MMFM field.
\item \textbf{History as source + context}: the same prior is the source, and
the same history summary additionally conditions the field as in
\texttt{EHRFlow}. 
\item \textbf{CFM (adjacent)}: conditional flow matching with the linear
interpolant on pairs of consecutive observations, with no multi-observation path
construction.
\item \textbf{CFM (adjacent) + history}: the same model conditioned on the
frozen, pretrained history representation $h_a$, exactly as in \texttt{EHRFlow}.
\item \textbf{CFM (MMFM windows) + history}: the same history-conditioned
model trained on the anchor--target pairs used by MMFM, but with a linear
interpolant that ignores intermediate observations. This is closely related
to the history-conditioned pairwise formulation of trajectory flow matching
\citep{zhang2024trajectory}, but omits its uncertainty-prediction component,
whose interpretation in observation space would require propagating latent
uncertainty through the decoder.

\end{itemize}

\subsection{Treatment-effect experiment}
\label{app:causal}

\paragraph{Cohort.}
A Synthea population of 20{,}000 living patients aged 40--90 is generated
once with the stock modules (treatment arm) and once with antihypertensive
prescriptions suppressed (control arm), yielding 26{,}390 and 26{,}879
patient records including deceased patients (66.6M and 64.0M events).
Because Synthea prescribes antihypertensives to every hypertensive patient,
we construct the control arm by suppressing prescriptions, obtaining
paired simulated trajectories under both treatment conditions.

\paragraph{Estimand and evaluation designs.}
Treatment is defined as the first antihypertensive prescription.
The target outcome is the last systolic blood-pressure reading within
1{,}825 days of treatment initiation, requiring at least two outcome
readings. The simulator-defined ground-truth effect is $-20.48$~mmHg
(sd 14.45; $n=19{,}618$ paired patients).
We evaluate two designs: \emph{Paired}, in which both treatment conditions
are available for each patient and effects are averaged over treated
patients; and \emph{Split-arm}, in which each patient contributes only
one treatment condition, assigned by a hash of the patient identifier.

\paragraph{Guided \texttt{EHRFlow}.}
The treatment arm enters the vector field through a 32-dimensional
embedding, replaced by a learned null token with probability 0.1 during
training for classifier-free guidance \citep{Ho2022}.
At inference, the guided field
$v_\varnothing + w\,(v_a - v_\varnothing)$ is integrated with 16 steps.
The guidance weight is selected from 14 values in $[0.5, 3.0]$ by
minimising validation mean squared error between the rolled-out outcome
and each patient's observed factual reading, without access to the
counterfactual ground truth.

\subsection{Computational resources}
\label{app:compute}

All main experiments, including the synthetic benchmarks, ablations, and treatment-effect experiments, were run on a machine with a single NVIDIA RTX 6000 Ada GPU (48~GB) and two AMD EPYC 7713 64-core processors, inside a container limited to 32 CPU threads and a 120~GiB memory cgroup (Ubuntu 22.04, Python 3.13, PyTorch 2.7.1, CUDA 12.6). Table~\ref{tab:compute} reports wall-clock time per run for the synthetic datasets (training and evaluation). The CPRD and AMC-EHR experiments were run on separate specialised compute environments that complied with the respective data-access and security requirements of the data providers. Unlike the synthetic datasets, these datasets did not fit entirely in memory, so their runtimes were substantially influenced by I/O bottlenecks, with individual runs taking several days. All synthetic experiments, including Synthea and the ablations, were repeated with five random seeds; CPRD and AMC-EHR experiments used three and five seeds, respectively, and treatment-effect recovery experiments used ten seeds.

\begin{table}[t]
\vspace{-12pt}
\centering
\caption{Wall-clock time in hours for each of the synthetic datasets (mean
$\pm$ std over seeds).}
\label{tab:compute}
\resizebox{\textwidth}{!}{%
\begin{tabular}{lcccccc}
\toprule
Wall-clock (h) & Anthracycline & RA on DMARD & CKD & RRMS & Heart failure & Synthea \\
\midrule
Decoder-only & 56.94 $\pm$ 26.02 & 47.02 $\pm$ 20.78 & 51.38 $\pm$ 26.94 & 52.84 $\pm$ 24.04 & 50.75 $\pm$ 22.82 & 23.36 $\pm$ 11.54 \\
Encoder-decoder & 11.41 $\pm$ 4.18 & 12.81 $\pm$ 4.68 & 13.82 $\pm$ 4.66 & 10.02 $\pm$ 4.52 & 15.24 $\pm$ 5.64 & 14.34 $\pm$ 4.52 \\
GRU-ODE-Bayes & 15.95 $\pm$ 2.70 & 9.78 $\pm$ 1.12 & 12.40 $\pm$ 1.53 & 12.86 $\pm$ 2.06 & 13.00 $\pm$ 1.76 & 13.78 $\pm$ 1.81 \\
Neural CDE & 4.32 $\pm$ 0.92 & 3.70 $\pm$ 0.39 & 4.37 $\pm$ 0.51 & 4.37 $\pm$ 0.48 & 4.39 $\pm$ 0.43 & 12.72 $\pm$ 3.78 \\
CFM (adjacent) & 1.15 $\pm$ 0.28 & 0.27 $\pm$ 0.00 & 0.29 $\pm$ 0.01 & 0.52 $\pm$ 0.05 & 0.30 $\pm$ 0.01 & 1.57 $\pm$ 1.92 \\
Metric FM & 1.31 $\pm$ 0.20 & 0.87 $\pm$ 0.19 & 1.13 $\pm$ 0.05 & 0.57 $\pm$ 0.09 & 1.11 $\pm$ 0.15 & 0.63 $\pm$ 0.02 \\
GP-CFM & 0.88 $\pm$ 0.39 & 0.48 $\pm$ 0.03 & 0.51 $\pm$ 0.03 & 0.80 $\pm$ 0.32 & 0.88 $\pm$ 0.42 & 0.73 $\pm$ 0.03 \\
MMFM & 0.55 $\pm$ 0.40 & 0.21 $\pm$ 0.02 & 0.22 $\pm$ 0.04 & 0.54 $\pm$ 0.34 & 0.63 $\pm$ 0.42 & 0.45 $\pm$ 0.03 \\
\midrule
\texttt{EHRFlow} (GP-CFM) & 0.66 $\pm$ 0.24 & 0.40 $\pm$ 0.01 & 0.43 $\pm$ 0.02 & 0.61 $\pm$ 0.19 & 0.68 $\pm$ 0.28 & 0.76 $\pm$ 0.02 \\
\texttt{EHRFlow} (MMFM) & 0.36 $\pm$ 0.20 & 0.18 $\pm$ 0.02 & 0.19 $\pm$ 0.04 & 0.35 $\pm$ 0.17 & 0.41 $\pm$ 0.29 & 0.50 $\pm$ 0.01 \\
\bottomrule
\end{tabular}}
\vspace{-12pt}
\end{table}

\section{AI use statement}
We used large language models (LLMs) to assist with improving the clarity of the manuscript, writing code for the experiments and refining the formatting of tables and figures. LLMs were not used for research ideation or any other substantive contributions that would merit authorship. The authors take responsibility for the final content of this work.

\end{document}